\documentclass{article}

\usepackage{microtype}
\usepackage{graphicx}
\usepackage{booktabs} %

\usepackage{amsmath,amsfonts,bm}

\renewcommand{\Vec}{\mathrm{Vec}}

\def\eqref#1{equation~\ref{#1}}
\def\Eqref#1{Equation~\ref{#1}}

\def\1{\bm{1}}

\DeclareMathAlphabet{\mathsfit}{\encodingdefault}{\sfdefault}{m}{sl}
\SetMathAlphabet{\mathsfit}{bold}{\encodingdefault}{\sfdefault}{bx}{n}

\newcommand{\E}{\mathbb{E}}

\newcommand{\R}{\mathbb{R}}

\DeclareMathOperator{\tr}{tr}

\usepackage{pgfplots}
\usepackage{diagbox}
\usepackage{hyperref}
\usepackage{url}
\usepackage{amssymb}
\usepackage{amsthm}
\usepackage{subcaption}
\usepackage{listings}
\usepackage{multirow}
\usepackage{listings}
\usepackage{xcolor}
\usepackage{beramono} %
\usepackage{url}
\usepackage[disable,textwidth=1.1in,textsize=tiny]{todonotes}
\pgfplotsset{compat=1.18}
\definecolor{shellbackground}{rgb}{0.98, 0.98, 0.98}

\usepackage{amsmath}

\usepackage[accepted]{icml2026}

\makeatletter
\newcommand{\ICML@appearing}{\vspace{-2em}}
\makeatother

\usepackage{mathtools}

\usepackage[capitalize,noabbrev]{cleveref}

\theoremstyle{plain}
\newtheorem{theorem}{Theorem}[section]

\newtheorem{lemma}[theorem]{Lemma}
\newtheorem{corollary}[theorem]{Corollary}
\theoremstyle{definition}
\newtheorem{definition}[theorem]{Definition}

\theoremstyle{remark}

\usepackage[textsize=tiny]{todonotes}
\newcommand{\AG}[1]{\todo[color=blue!30]{\textbf{AG:} #1}} %
\newcommand{\RG}[1]{\todo[color=purple!30]{\textbf{RG:} #1}} %

\newcommand{\norm}[1]{\left\lVert #1 \right\rVert}
\newcommand{\Hoff}{\|H\|_{\text{off}}}
\newcommand{\UB}{\mathrm{UB}}
\newcommand{\OptRotH}{OptRot$^{+}$}
\icmltitlerunning{OptRot}

\begin{document}

\twocolumn[
\icmltitle{OptRot: Mitigating Weight Outliers via Data-Free \\ Rotations for Post-Training Quantization}

\icmlsetsymbol{equal}{*}
\icmlsetsymbol{msr}{\dag}
\begin{icmlauthorlist}
\icmlauthor{Advait Gadhikar}{equal,msr,yyy}
\icmlauthor{Riccardo Grazzi}{equal,sch}
\icmlauthor{James Hensman}{sch}
\end{icmlauthorlist}

\icmlaffiliation{yyy}{CISPA Helmholtz Center for Information Security, Germany}
\icmlaffiliation{sch}{Microsoft Research, Cambridge, UK}

\icmlcorrespondingauthor{Advait Gadhikar}{advait.gadhikar@cispa.de}

\icmlkeywords{Machine Learning, ICML}

\vskip 0.3in
]

\printAffiliationsAndNotice{\textsuperscript{$*$} Equal contribution.\textsuperscript{$\dag$}This research was performed while the author was at Microsoft Research Cambridge.}  %

\begin{abstract}
The presence of outliers in Large Language Models (LLMs) weights and activations makes them difficult to quantize. Recent work has leveraged rotations to mitigate these outliers. 
In this work, we propose methods that learn fusible rotations by minimizing  principled and cheap proxy objectives to the weight quantization error. We primarily focus on GPTQ as the quantization method.
Our main method is OptRot, which reduces weight outliers simply by minimizing the element-wise fourth power of the rotated weights. 
We show that OptRot outperforms both Hadamard rotations and more expensive, data-dependent methods like SpinQuant and OSTQuant for weight quantization. It also improves activation quantization in the W4A8 setting. 
We also propose a data-dependent method, \OptRotH, that further improves performance by incorporating information on the activation covariance.
In the W4A4 setting, we see that both OptRot and \OptRotH  perform worse, highlighting a trade-off between weight and activation quantization.
\end{abstract}

\section{Introduction}
Increasing model size has enabled LLMs to perform a range of tasks \citep{achiam2023gpt, grattafiori2024llama3, yang2025qwen3}, encouraging practitioners to design huge models with billions of parameters.
Post Training Quantization (PTQ) is a widely adopted strategy for compressing these models by lowering their precision, while limiting the drop in performance to enable efficient inference \citep{xu2024scaling}.
Scalar quantization is the most popular PTQ approach and relies on mapping each parameter value to a point on a finite precision grid, determined by the bit-width.
Outliers prevent the finite grid from uniformly covering all values and can lead to large quantization errors.
Hence, outlier reduction has increasingly received attention as a crucial pre-processing step \citep{chee2023quip} for PTQ with algorithms like simple Round-to-Nearest (RTN) or the more powerful GPTQ \citep{frantar2022gptq}.

Recent work has applied rotations to weight matrices to mitigate outliers in weights and activations while keeping the network functionally equivalent. 
Hadamard rotations have shown to significantly improve the performance of GPTQ \citep{ashkboos2024quarot, tseng2024quip}.
To minimize overhead, these rotations can be materialized online efficiently with the Walsh-Hadamard transform, and in some cases they can also be fused with model weights.
Fused rotations can also be learned without increasing inference cost.
Learning rotations significantly improves performance.
Heuristics like making the weight or activation distribution more uniform \citep{akhondzadeh2025kurtail, shao2025dartquant} have shown to learn better rotations.
Further improvements have been achieved by minimizing the loss of the model when rotated and then quantized in the forward pass \citep{liu2024spinquant}. However, this quantization-aware training phase uses exclusively RTN as the quantization method, since GPTQ would be too slow.

In this paper, we design efficient objectives to learn rotations that provably improve the quantization error, without quantization-aware training. %
To do so, we closely examine the GPTQ quantization objective, which minimizes a layerwise approximation of the KL divergence between the quantized and original model.
A smaller KL divergence, implies better quantization -- the quantized model is closer to the original. 
Leveraging the theoretical framework proposed by \citet{chee2023quip}, we show that the GPTQ objective has an upper bound which depends on the weight incoherence (which captures the degree of weight outliers).

We propose OptRot, a data-free method that learns rotations by minimizing a smooth proxy for the weight incoherence of the model: the element-wise fourth power of the rotated weights.
We show that OptRot achieves better incoherence and weight quantization performance than QuaRot, SpinQuant, OSTQuant and QuIP$\#$.
Different to prior work, OptRot is data-free and (except Kurtail) quantization-agnostic: it does not require a calibration set or the quantization scheme to learn rotations.
OptRot can also learn rotations by optimizing a subset of the weights, without loading the entire model on GPU.

We also derive an upper bound to the GPTQ error that, in addition to the weight incoherence, also depends on a quantity related to the amount of feature correlation captured by the covariance matrix of the activations (or Hessian), which can be changed via rotations. 
Higher feature correlations makes quantizing with GPTQ easier. 
To this end, we introduce \OptRotH, a data-dependent extension of OptRot, which learns rotations that jointly optimize the weight incoherence and Hessian feature correlation. %
\OptRotH trades off data independence for smaller KL divergence of the quantized model.

Our contributions can be summarized as follows:

\begin{itemize}
    \item We derive cheap proxy objectives to learn rotations that improve weight quantization error. We do so by extending the theoretical framework of \citet{chee2023quip} and deriving error bounds for quantization with GPTQ and round-to-nearest (RTN).
    \item We propose OptRot, a data-free method which learns rotations that reduces the weight incoherence by minimizing the elementwise fourth power of weights. We also propose a data-dependent version, \OptRotH, that further improves performance.
    \item OptRot outperforms methods like SpinQuant, QuaRot and OSTQuant for weight-only quantization and is competetive with SpinQuant for activation quantization with A8W4.
     We also observe a trade-off for activation quantization with A4W4, where reducing weight incoherence worsens activation quantization.
\end{itemize}

\begin{figure*}
    \centering
    \includegraphics[width=0.95\linewidth]{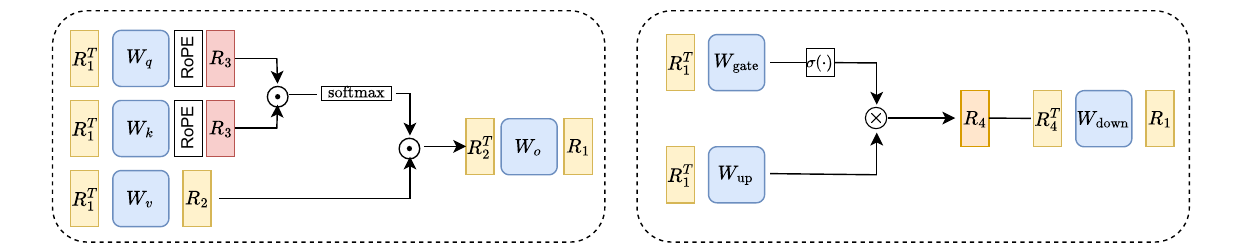}
    \caption{
    Rotations applied by QuaRot, SpinQuant and OptRot to improve PTQ.
    $R_1, R_1^\top, R_2, R_2^\top$ and $R_4^\top$ (yellow) are fusible rotations. 
     $R_3$ (red) and $R_4$ (orange) are online rotations. $R_1$ is shared across layers. In OptRot and SpinQuant  $R_1$ and $R_2$ are learned without increasing inference cost.}
    \label{fig:rotations}
\end{figure*}

\section{Related Work}

GPTQ \citep{frantar2022gptq} is the most commonly used scalar quantization method for weight quantization, strongly outperforming the round-to-nearest (RTN) baseline.
Outliers, present in weights and activations can significantly impact the quantization performance of GPTQ (and RTN).
Various approaches have tried to tackle outliers in this regard.
\citet{dettmers2022gpt3} suggest a mixed precision method where outliers are maintained in high precision high precision.
SmoothQuant \citep{xiao2023smoothquant} transfers activation outliers to weights.
MagR \citep{zhang2024magr} introduces a regularizer to penalize outlier weights and learns an invertible map.

More recently, methods have leveraged the rotation invariance in LLMs to rotate weights and activations to mitigate outliers without changing model function.
QuaRot \citep{ashkboos2024quarot} and QuIP\# \citep{tseng2024quip} use Hadamard rotations to this effect.
\citet{akhondzadeh2025kurtail} learns rotations that make activations more uniformly distributed.
Similarly, DartQuant \citep{shao2025dartquant} learns rotations that encourage activation uniformity and maximize space utilization respectively. 
SpinQuant~\citep{liu2024spinquant} and Kurtail~\citep{akhondzadeh2025kurtail} learn fused rotations to improve mainly activation quantization, while FlatQuant \cite{sun2024flatquant} learns invertible linear maps which cannot be fused, instead of rotations, to reduce outliers. 
 ButterflyQuant~\citep{xu2025butterflyquant} learns orthogonal transforms parameterized by Givens rotations. 
 OSTQuant \citep{ostquant} learns rotations and invertible scalings that minimize the KL divergence between the quantized and orignal model.
 SpinQuant, OSTQuant and FlatQuant are tailored for quantization via RTN, as activations (or weights) are quantized with RTN during rotation optimization.
 In contrast, this work focuses on learning transformations that improve quantization with GPTQ, which is more challenging and impractical with end-to-end methods. 

\section{Quantization Error and Outliers}
\label{sec:theory}

We first outline the theoretical framework and establish a connection between the quantization error and the weight and activation outliers, in the case of weight quantization.
Our goal is to learn rotations that minimize the KL divergence between the quantized and original model: 
\begin{align}
    \min_{\mathcal{R}} \big\{\ell(\theta, \hat \theta) :=\mathbb{E}_z   \mathrm{KL} (p_\theta(\cdot \mid  z) \,\|\ p_{\hat \theta}( \cdot \mid z))\big\}\label{eq:kl}\,,
\end{align}
where $    \mathrm{KL}(p \,\|\ q) = \mathbb{E}_{y \sim p} \left[\log \left(p(y)/q(y)\right)\right]
$ and $p_\theta(\cdot \mid  z)$ denotes the output probability distribution of a language model given an input sequence $z$. The original weights are denoted by $\theta$ while the rotated and then quantized weights with $\hat \theta$. The set of rotations (orthogonal matrices) is denoted by $\mathcal{R}$.
For the llama models, rotations are applied to weights as shown in ~\cref{fig:rotations}. 

Minimizing (\ref{eq:kl}) directly with a gradient based method can be challenging: it requires to compute several forward and backward passes through the model and to differentiate through the quantization operation, which is piecewise constant and can be slow in the case of GPTQ. Our goal is to design principled, cheaper proxy objectives that still allow us to minimize (\ref{eq:kl}).

We begin by approximating the KL divergence between the quantized and original model as the sum of the layerwise reconstruction errors (details are provided in Appendix~\ref{app:approx-kl}):
\begin{align} \label{eq:kl-approx}
    \ell(\theta, \hat \theta) &\approx 
     \sum_{W,H} \mathcal{L}(\hat{W} - W, H) 
\end{align}
\begin{equation}
\begin{aligned}
\mathcal{L}(\hat{W} - W, H) 
    &= \tr \big( (\hat{W} - W) H (\hat{W} - W)^\top  \big) \\
    &=  \E_x \Big [\norm{(\hat{W} - W)x }^2 \Big]\, .
\end{aligned}
 \label{eq:trace-obj}   
\end{equation} 
Here, $W \in \R ^{m \times n}$ is a weight matrix of the original LLM  and $\hat{W}$ denotes the rotated and then quantized weight. 
 $x \in \R^n$ is the input activation from the previous layer, and $H = \E_x[xx^\top] \in \R^{n \times n}$ is the uncentered input covariance or Hessian. 
The term $\mathcal{L}(W-\hat W, H)$ measures the quality of a quantized weight $\hat W$ in isolation --  how well we can recover the outputs of the original layer.

Following \citet{chee2023quip}, we define the incoherence of $W$ and $H$ with SVD $H = Q \Lambda Q^\top $ as 
\begin{align*}
    \mu_W :&= \sqrt{mn} \,w_{\mathrm{max}} / \lVert W\rVert_F\,, \quad &w_{\mathrm{max}} &:= \max_{ij} \lvert W_{ij} \rvert\,,\\
    \mu_H :&= \sqrt{n}\, q_{\mathrm{max}}\,,  &q_{\mathrm{max}} &:=\max_{i,j} \lvert Q_{ij}\rvert \,.\nonumber\label{eq:incoherence-defn}
\end{align*}
The weight incoherence, $\mu_W$, measures the tailed-ness of the distribution (which increases with more outliers). 
It has minimum  $\mu_W = 1$, which is achieved when all weights are constant ($W_{i,j} = c \in \R$).
The Hessian incoherence $\mu_H$ has a similar role and also has minimum $1$.  
As we will show later, low incoherence improves quantization error.

 QuIP\# \cite{tseng2024quip} exploits incoherence processing, which decreases $\mu_W$ and $\mu_H$ by transforming $\tilde{W} \leftarrow UWV^\top $ and $\tilde{H} \leftarrow VHV^\top $ before quantization, where $U$, $V$ are random Hadamard matrices. This approach introduces overhead during inference since the activations must be rotated prior to matrix multiply. QuaRot also uses Hadamard rotations, but exploits the rotational invariance of LLMs to reduce the overhead by fusing some of the rotations with the weights. These fusible rotations can be learned without additional overhead during inference, as done e.g.\@ in SpinQuant. We provide a schematic of the rotations used in \Cref{fig:rotations}: we learn two fusible rotations ($R_1$ and $R_2$).

\subsection{Quantization Error Bounds} 

We now describe two popular quantization methods, round-to-nearest (RTN) and GPTQ, for scalar quantization and derive upper bounds for the quantization error in both cases.

\textbf{RTN.} A simple strategy to quantize a weight matrix is $b$-bits RTN. First, each element of the weight matrix is scaled by applying the function  $g(x; s) =  \frac{2^b-1}{2}\left(\frac{x}{s} + 1\right)$ (assuming symmetric quantization), where $s > 0$ is a scale parameter. Second, it is mapped to the nearest value of the quantization grid via the function $h(x) = \arg\min_{c \in \{0, \dots,2^b-1\}} |x-c| $, where we considered integer quantization. The resulting values are then stored and during inference $g^{-1}$ is applied to retrieve the approximate weights $\hat W$.  Combining all steps we obtain 
\begin{equation}\label{eq:RTN}
  \hat W_{\mathrm{RTN}} := Q(W ; s) := g^{-1} (h(g(W; s)); s)  \,,
\end{equation}
where we extended $g,g^{-1},h$ to element-wise functions applied to the entire matrix. The scale $s$ can be set to $w_{\max}$ so that all weights are inside the quantization grid, but the max is usually computed and applied over a small contiguous group of weights to mitigate outliers with a small increase in memory (as more scales have to be stored). 
The following theorem bounds the layerwise reconstruction error of RTN. Proof is in Appendix~\ref{app:RTN}.
\begin{theorem}[Worst Case Error Bound for RTN]\label{thm:rtnbound}
Let $W \in \R^{m \times n}$ be a weight matrix and $H \in \R^{n \times n}$ be a symmetric positive semi-definite (PSD) matrix. If $s = w_{\mathrm{max}}$, then 
\begin{equation}\label{eq:rtn-inc-bound}
\mathcal{L}(\hat W_{\mathrm{RTN}}-W, H) \leq \frac{\mu_W^2}{(2^b-1)^2} \lambda_{\max}(H)  \norm{W}^2_F     \nonumber
\end{equation}
where $\mu_W$ is the weight incoherence and $\lambda_{\max}(H)$ is the maximum eigenvalue of the matrix $H$.
\end{theorem}

\textbf{GPTQ.} The GPTQ algorithm finds the quantized weight which approximately minimizes $\mathcal{L}$ by exploiting the LDL decomposition of the matrix $H$. In particular, as shown by \citet{chee2023quip}, it computes the fixed point 
\begin{align}\label{eq:gptq-method}
\hat W_{\mathrm{GPTQ}} =  Q(W  + (W - \hat W_{\mathrm{GPTQ}}) U)
\end{align}
where $U$ is the strictly upper triangular matrix from the LDL decomposition of the hessian $H = (U + I)D(U + I)^\top$.
Obtaining a proper worst-case bound for GPTQ is more challenging than for RTN since the corrections  applied before discretization via $Q$ can push some values far outside the quantization grid and can yield a large error after clamping. Indeed, as shown in \citet[Section 5.2]{chee2023quip}, if the weight matrix is carefully constructed, RTN can greatly outperform GPTQ although in practice that's rarely the case.

To address this issue and theoretically analyze the GPTQ error, \citet{chee2023quip} used a modified version of GPTQ, which we name GPTQS, that uses a constrained LDL of the Hessian and stochastic rounding to ensure that all corrected weights lie within the quantization grid with high probability.
\begin{definition}[Constrained LDL]\label{def:constrained-ldl}
    Given a Hessian matrix $H \in \mathbb{R}^{n \times n}$, we define the constrained LDL, $L$, as the solution to the optimization problem,
\begin{align*}
    \operatorname{minimize}&: \tr(HL^\top L) \\
    \operatorname{over}&: L \text{ unit upper triangular} \\
    \operatorname{subject~to}&: e_i^\top L^\top Le_i \le 1 + c,\, \forall i \in \{1,...,n\}. 
\end{align*}
\end{definition}
Note that if $c$ is large enough, then the constrained LDL is equal to the true LDL, i.e.\@ $L^{-1} = U+I$.  
The modified GPTQ is then
\begin{align}\label{eq:gptqs}
    \hat W_{\mathrm{GPTQS}} = \hat Q(W  + (W - \hat W_{\mathrm{GPTQS}}) (L^{-1}-I); s)
\end{align}
where $s = w_{\text{max}}$, $\hat Q(W ; s) = \hat g^{-1} (\hat h(\hat g(W; s)); s)$,
with $\hat g(x; s) =  \frac{2^b-3}{2}\left(\frac{x}{s} + 1\right) + 1$ and
\begin{equation*}
 \hat h(x) = \begin{cases} 
\lceil \tilde{x} \rceil & \text{w.p. } \tilde{x} - \lfloor \tilde{x} \rfloor \\ 
\lfloor \tilde{x} \rfloor & \text{w.p. } \lceil \tilde{x} \rceil - \tilde{x}. 
\end{cases}, 
\quad \tilde{x} = \min(\max(x, a), b)   
\end{equation*}
denotes stochastic rounding with $a=0,b=2^b-1$ being the limits of the quantization range.
The following theorem bounds the GPTQS error in terms of the weight and Hessian incoherence. The value of $c$ is chosen such that the correction term can be properly controlled, i.e. with probability at least $1-\delta$  it does not exceed $1$ for any weights as per \citep[Lemma 13]{chee2023quip}.
\begin{theorem}[Theorem 14 in \citet{chee2023quip}] \label{thm:gptq-bound-quip}
Let $H$ be the Hessian, and $L$ be the constrained LDL as defined in Definition \ref{def:constrained-ldl} with $c = 2\left(\log (\frac{4mn}{\delta})\right)^{-1}$.
Then the quantization error of $\hat W_{\mathrm{GPTQS}}$ is bounded with probability at least $1-\delta$ as:
    \begin{align}
     &\mathcal{L}(\hat W_{\mathrm{GPTQS}}-W, H) \nonumber\\
     &%
     \leq \frac{\mu_H^2 \mu_W^2}{n^2(2^b-3)^2} \tr \left(H^{1/2}\right)^2  \norm{W}^2_F  \log \Big(\frac{4mn}{\delta}\Big)^2 \,. 
     \label{eq:gpt-inc-bound}
\end{align}
\end{theorem}
In this work, we build on the following modified (tighter) version of this theorem which depends directly on the constrained LDL.
\begin{theorem}[]\label{thm:gptq-bound-trace}
Let $H$ be the Hessian, and $L$ be the constrained LDL as defined in Definition \ref{def:constrained-ldl} with $c = 2\left(\log (\frac{4mn}{\delta})\right)^{-1}$.
Then the quantization error or $\hat W_{\mathrm{GPTQS}}$ is bounded with probability $1-\delta$ as
\begin{align}
&\mathcal{L}(\hat W_{\mathrm{GPTQS}}-W, H) \nonumber\\
&\le \frac{\mu_W^2}{n(2^b-3)^2}\tr(HL^\top L)\norm{W}_F^2 \cdot \frac{1}
    {2}\log\left(\frac{2mn}{\delta}\right)
\end{align}
\end{theorem}

We can recover \Cref{thm:gptq-bound-quip} from \Cref{thm:gptq-bound-trace} by applying the following inequality, also proven in \citet{chee2023quip}.
\[
\tr\left(HL^\top L\right) \le \frac{1}{\min(1, c)}\mu_H^2 \tr\left(H^{1/2}\right)^2/n.
\]
However, we note that the Hessian incoherence $\mu_H$ has a problematic non-smooth dependence on the coordinates of the eigenvectors of $H$ and can actually be ill-defined when two or more eigenvalues coincide.
For this reason, to get a bound which is more robust and suitable for optimization, we rely on the different inequality (proven in Lemma~\ref{thm:trace-bound} in the Appendix):
\begin{equation}\label{eq:smoothbound}
\tr\left(HL^\top L\right) \le 2\left( \tr(H) - \frac{\Hoff^2}{2\tr(H)} \right) =: 2\UB,    
\end{equation}
where $\Hoff^2 = \sum_{i \neq j} |H_{ij}|^2$ varies with rotations. We compare this bound against the incoherence bound in Appendix~\ref{app:bound_compare}. This  yields the following corollary.

\begin{corollary}[GPTQ Bound]\label{remark:gptq-bound-trace-d}
The quantization error of $\hat W_{\mathrm{GPTQS}}$ with $c = 2\left(\log (\frac{4mn}{\delta})\right)^{-1}$ is bounded w.p. $1-\delta$ as
    \begin{align}
&\mathcal{L}(\hat W_{\mathrm{GPTQS}}-W, H)\nonumber\\
     &\le \frac{2\mu_W^2}{n(2^b-3)^2}\UB \, \norm{W}_F^2 \cdot\frac{1}
    {2}\log\left(\frac{2mn}{\delta}\right) 
\end{align}
\end{corollary}

\textbf{Improving error bounds with rotations.} 
Applying rotations to $W$ or $H$ only affects the error upper bounds via the weight incoherence $\mu_W$ for RTN and GPTQ, the hessian incoherence $\mu_H$ (via $w_{\mathrm{max}}$) or $\UB$ (via $\Hoff$) for GPTQ. In the following section we design proxy objectives that find rotations to minimize the error upper bounds.

\section{OptRot}\label{sec:optrot}
In this section we design the rotation learning problems solved by OptRot and \OptRotH by combining \Cref{eq:kl-approx} with approximations to the upper bound derived in the previous section. In particular, the final optimization problems will be of the form
\begin{align}\label{eq:optrot-obj}
    \min_{R_1, R_{2,1},\dots, R_{2,L}}\sum_{l,s} l_{\text{rot} }(\tilde W^{(l,s)}, \tilde H^{(l,s)})
\end{align}
where $\tilde W, \tilde H$ are the rotated weight and Hessian matrices and $l_{\text{rot}}$ is a suitable approximation of the GPTQ bound in Corollary~\ref{remark:gptq-bound-trace-d}.
For llama models, the weights are rotated as per ~\Cref{fig:rotations}, $ \tilde W^{(l,s)} = R_1^\top W^{(l,s)} $for $s \in \{\text{q},\text{k}, \text{gate}, \text{up}\}$ and $ \tilde W^{(l,\text{v})} = R_1^\top W^{(l,\text{v})}R_{2,l}, \tilde W^{(l,\text{o})} = R_{2,l}^\top W^{(l,\text{o})}R_{2,l}, \tilde W^{(l,\text{down})} = R_4^\top W^{(l,\text{down})}  R_1.$
For the optimization, we use Cayley-SGD\citep{li2020efficient}, which ensures that at each step the rotations are inside the Stiefel Manifold.

Focusing on a single layer and using the definition of weight incoherence, where $w_{\text{max}} = \max_{i,j}|W_{i,j}|$, we can rewrite the bound  in \Cref{thm:gptq-bound-trace} as
\begin{align}\label{eq:layerwise-error-bound}
\mathcal{L}(W - \hat W, H) \le \frac{w_{\text{max}}^2\tr\left(HL^\top L\right)}{(2^b-3)^2} \cdot\frac{m}
    {2}\log\left(\frac{2mn}{\delta}\right).
\end{align}

Since we keep the quantization bit-width $b$ constant for all layers, we can remove it from our objective.
For simplicity, we also ignore the term $\frac{m}{2}\log\left(\frac{2mn}{\delta}\right)$, where $m$ is the output dimension and $n$ is the input dimension of the weight matrix, arriving at 
\begin{align}\label{eq:joint-obj}
  \ell_{\operatorname{rot}} = w_{\text{max}}^2 \tr(HL^\top L)\,.
\end{align}

\subsection{Data-Free Objectives - OptRot}
We begin with introducing OptRot, by choosing the simplest data-free objective.
OptRot focuses on the weight outliers, ignoring the $\tr(HL^\top L)$ term.
Instead of directly minimizing $l_{\text{rot}} = w_{\text{max}}$,  we replace the non-smooth $ w_{\text{max}}$ with $\|\operatorname{vec}(\tilde W)\|_p^p = \sum_{i,j} |\tilde W_{i,j}|^p$ where $\operatorname{vec}(\cdot)$ flattens matrices into vectors, and a high $p$-norm provides a smooth approximation to the infinity norm. Concretely, we solve the following problem over our smooth incoherence loss:
\begin{align}\label{eq:incobj}
    \min_{R_1, R_{2,1},\dots, R_{2,L}} \sum_{l,s} l_{\text{rot}}, \: \qquad \, l_{\text{rot}}=\|\operatorname{vec}(\tilde W^{(l,s)})\|_p^p
\end{align}

We choose $p=4$ in our experiments as larger $p$ showed no improvement in preliminary results and is related to the kurtosis, another measure of outliers \citep{akhondzadeh2025kurtail}.
We also propose an alternative objective, OptRot-v2, which uses the squared $p$ norm as the smooth objective. 
In this case $l_{\text{rot}}=\|\operatorname{vec}(\tilde W^{(l,s)})\|_p^2$, which matches the bound more closely.

\subsection{Data-Dependent Objectives - \OptRotH}

We now look at jointly the sum of $w_{\text{max}}^2 \tr(HL^\top L)$ over all linear layers. 
The term $\tr(HL^\top L)$ is a data-dependent term and requires a calibration set to compute the Hessian $H$. After that, as done in practice for GPTQ, we can use the LDL decomposition instead of the constrained LDL and compute $\tr(D)$ in place of $\tr(HL^\top L)$ via the Cholesky decomposition, which supports backpropagation in PyTorch.
While this introduces additional cost and data dependence, this step is also necessary for quantizing with GPTQ downstream.
However, while GPTQ uses the Hessians to do a one-time Cholesky decomposition, optimizing the joint objective would require to backpropagate through the Cholesky decomposition at every step, since the Hessian will change with the rotation. Furthermore, approximating the constrained LDL ($c = 2\left(\log (\frac{4mn}{\delta})\right)^{-1}$) with the LDL ($c = \infty$) makes the resulting bound not valid and indeed the constrained LDL is usually quite different from the standard LDL as we see in ~\Cref{fig:correction-bound} in the appendix.

Therefore, we also consider the cheaper and smooth upper bound  $\tr(HL^\top L) \leq 2\UB$ in \Eqref{eq:smoothbound}, that is still affected by rotations. The bound requires $O(n^2)$ operations and can be higly parallelized, while the Cholesky decomposition has complexity $O(n^3)$ and $n$ sequential steps.

A simpler  upper bound of $\tr(HL^\top L)$ is $\tr(H)$ (set $L = I$).
Since, $\tr(H)$ is rotation invariant, it only needs to be computed once and serves only as a layerwise importance score for the approximation of $w_{\text{max}}^2$.

Thus, our joint optimization objective, can be a combination of a bound on the $\tr(HL^\top L)$ or $\tr(D)$, and its product with a smooth approximation of $w_{\text{max}}$, summed over all layers.
For $w_{\text{max}}$, we consider the $p$-norm of the flattened weights $\|\operatorname{vec}(\tilde W)\|_p$ and its square, as it appears in the bound. Similar to OptRot, we choose $p=4$.

We perform a sweep over the six resulting joint objectives on a Llama-3.2-1B model to find the best performing ones.
We report in ~\cref{fig:multi-obj} the KL divergence between the model quantized with GPTQ and the original one, after learning rotations with each objective.
\begin{figure}[ht!]
    \centering
    \includegraphics[width=0.85\linewidth]{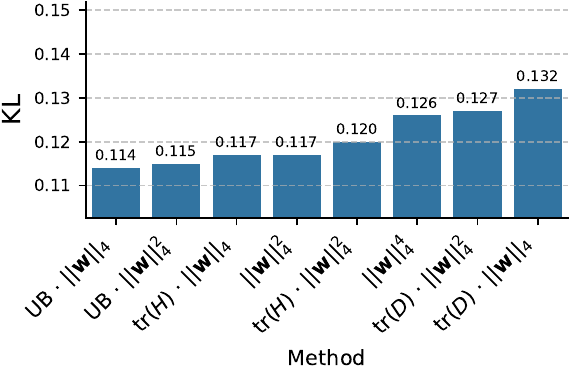}
    \caption{Comparing the KL divergence of data-dependent objectives after learning rotations and $4$-bit weight quantization with GPTQ on Llama-3.2-1B.
    }
    \label{fig:multi-obj}
\end{figure}
We observe that rotations learnt by minimizing $\text{UB}\cdot\|\operatorname{vec}(\tilde W)\|_4$ and $\text{UB}\cdot\|\operatorname{vec}(\tilde W)\|_4^2$ obtain the smallest KL divergence after quantization and choose them as our joint objective, which we term \OptRotH (see ~\Cref{tab:method-optrot}).
Note that optimizing the joint objective from ~\Cref{eq:joint-obj} yields lower KL divergence than the data-free versions.
\begin{table}[ht]
\centering
\caption{Introducing all versions of OptRot.}
\label{tab:method-optrot}
\begin{tabular}{l p{2.5cm} c}
\toprule
\textbf{Method} & \textbf{Objective} ($l_{\text{rot}}$) & \textbf{Data-Free} \\
\midrule
OptRot & $\|\operatorname{vec}(\tilde W)\|_4^4$& Yes  \\
OptRot-v2 & $\|\operatorname{vec}(\tilde W)\|_4^2$& Yes  \\
\OptRotH & $\text{UB}\cdot\|\operatorname{vec}(\tilde W)\|_4$& No  \\
\OptRotH-v2 & $\text{UB}\cdot\|\operatorname{vec}(\tilde W)\|_4^2$& No  \\
\bottomrule
\end{tabular}
\end{table}

\begin{figure*}[ht!]
    \centering
    \includegraphics[width=\linewidth]{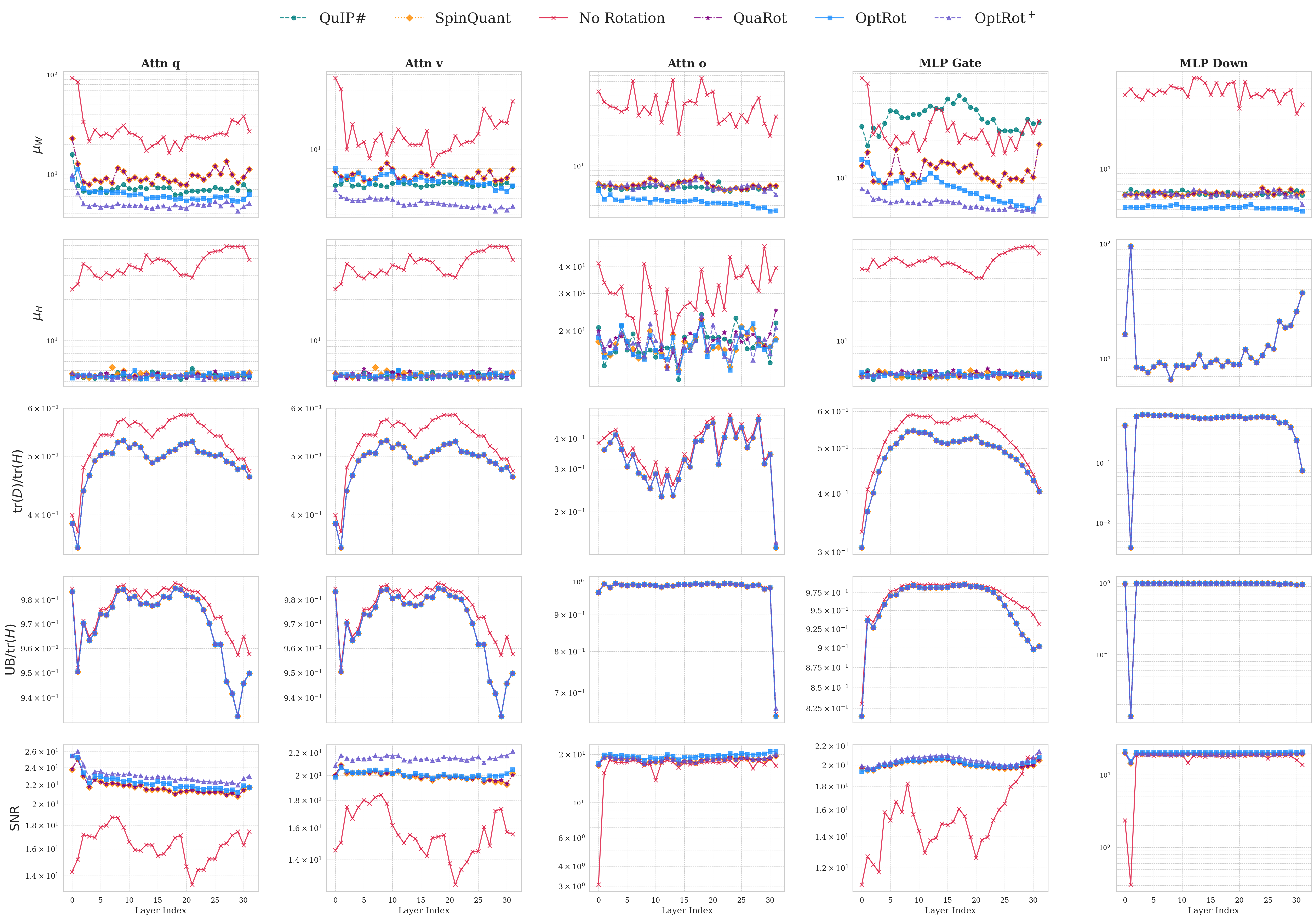}
    \caption{Weight incoherence $\mu_W$ optimized by OptRot (top row), Hessian incoherence $\mu_H$ (second row), $\tr(D) / \text{Tr}(H)$ (third row), the upper bound $\text{UB} / \text{Tr}(H)$ (fourth row) and the SNR after quantization with GPTQ (bottom row) for Llama-3.1-8B. 
    }
    \label{fig:inc-main}
\end{figure*}

\section{Experiments}
We conduct experiments on the Llama-3 and Qwen3 series of models, for weight only and weight + activation quantization.
Weights are quantized with GPTQ, calibrated on the C4 \cite{c4-dataset} dataset, activations are quantized with RTN.
We report results on Wikitext \citep{merity2017pointer} and six zero-shot commonsense reasoning benchmarks: Piqa \citep{bisk2020piqa}, Hellaswag \citep{zellers2019hellaswag}, Arc-E and Arc-C \citep{clark2018think}, Lambada \citep{paperno2016lambada} and Winogrande \citep{sakaguchi2020winogrande}. 
We also report the KL divergence, on the C4 dataset, between the quantized and original model.
Along with the methods outlined in ~\Cref{tab:method-optrot}, we report results with a cheaper version of OptRot, which chooses the top $50$ weights with the largest loss to learn rotations. 
Due to space constraints we defer some results to the appendix.

\textbf{Baselines.}
We give a brief outline of the baselines used.
\textbf{QuaRot} uses fusible Hadamard rotations and already achieves significant improvements over no rotations.
\textbf{QuIP\#} also uses Hadamard rotations, but applies two rotations (one to the left and the other to the right) to each weight matrix in an online manner instead of a single fusible rotation like QuaRot.
We only report the incoherence metrics for QuIP\# but do not quantize with it as the rotations are not fusible. 
\textbf{SpinQuant} learns rotations by performing Riemmanian gradient descent where in each forward pass activations are quantized with RTN. Therefore, rotations are learned to account only for the error in activation quantization, while GPTQ takes care of the weight quantization error. In theory, the GPTQ quantization step could also be included in the forward pass of the rotation learning phase, but this would make that phase much slower due to the sequential nature of GPTQ.
However, when we quantize weights using RTN, including weight quantization also in the rotation learning phase can be done and yields the best results for SpinQuant (as we show in Appendix~\ref{sec:rtn-results}).
\textbf{OSTQuant} is similar to SpinQuant but minimizes the KL divergence between the quantized and original model, instead of the cross-entropy loss of the quantized model,  and learns a linear invertible scaling in addition to rotations.
\textbf{DartQuant} proposes to use the Whip loss to mitigate activation outliers. Since we focus on weight quantization, we modify the method by applying the same loss not to the activations but to the weights -- denoted by DartQuant(W), as an alternative method to minimize weight outliers.

\begin{table*}[ht!]
    \centering
    \caption{Results for weight-only quantization at 4-bits with GPTQ.}
    \label{tab:weight-only}
    \setlength{\tabcolsep}{4pt} %
    \begin{tabular}{l@{\hskip 0.5cm}ccc@{\hskip 0.5cm}ccc@{\hskip 0.5cm}ccc}
        \toprule
        \multirow{2}{*}{\textbf{Method}} & \multicolumn{3}{c}{Llama-3.2-1B} & \multicolumn{3}{c}{Llama-3.2-3B} %
        & \multicolumn{3}{c}{Llama-3.1-8B} \\
        \cmidrule(lr){2-4} \cmidrule(lr){5-7} %
        \cmidrule(lr){8-10}
         & Acc $\uparrow$ & Wiki $\downarrow$ & KL $\downarrow$ & Acc $\uparrow$ & Wiki $\downarrow$ & KL $\downarrow$ %
        & Acc $\uparrow$ & Wiki $\downarrow$ & KL $\downarrow$ \\
        \midrule
        FP16 & $56.77$ & $9.76$& $0$ & $64.72$& $7.81$& $0$& $70.29$& $6.24$ & $0$\\
        \midrule
        No Rotation & $49.43$ & $13.86$ & $0.36$ & $56.28$ & $11.78$ & $0.35$ %
        & $67.39$ & $7.47$ & $0.23$ \\
        QuaRot & $55.26$ & $10.79$ & $0.136$ & $63.48$ & $8.36$ & $0.097$ %
        & $69.17$ & $6.72$ & $0.0926$ \\
        DartQuant(W) &$55.23$ & $10.79$& $0.131$ &$\textbf{63.92}$ &$8.36$ & $0.096$& $68.81$ & $7.21$&  $0.162$\\
        OSTQuant$^{*}$ & $49.25$ & $10.76$ & - & $62.29$ & $\textbf{8.28}$   & - & $69.14$ & $\underline{6.65}$ & -\\
        SpinQuant & $54.9$ & $10.73$ & $0.137$ & $63.05$ & $8.37$ & $0.096$ %
        & $69.46$ & $6.71$ & $0.0925$ \\
        \midrule
        OptRot (top-50) & $\underline{55.53}$ & $10.68$ & $0.123$ & $63.48$ & $8.32$ & $\underline{0.093}$ %
        & $69.5$ & $\underline{6.65}$ & $0.0849$ \\
        OptRot & $\textbf{55.65}$ & $10.62$ & $0.125$ & $63.41$ & $\textbf{8.28}$ & $\underline{0.093}$ %
        & $69.5$ & $\textbf{6.64}$ & $0.0866$\\
        \midrule
        OptRot-v2& $55.14$& $\underline{10.59}$ & $0.116$ & $\underline{63.89}$ & $8.33$& $0.99$ & $66.36$& $8.37$ & $0.264$\\
        \OptRotH & $55.07$& $\textbf{10.56}$ & $\textbf{0.114}$& $63.43$ & $\underline{8.3}$&$\textbf{0.086}$ & $\textbf{70.16}$ & $\underline{6.65}$ & $\textbf{0.08}$ \\
        \OptRotH-v2 & $55.23$& $10.63$ & $\underline{0.115}$ & $63.64$ & $8.32$& $\textbf{0.086}$& $\underline{69.74}$& $6.68$ & $\underline{0.083}$\\
        \bottomrule
    \end{tabular}
\end{table*}

\subsection{Layerwise Results}

Figure \ref{fig:inc-main} (top row) plots the weight incoherence $\mu_W$. 
Hadamard rotations with QuaRot improve the weight incoherence over no rotations.
SpinQuant does not further improve weight incoherence.
QuIP\# can improve it with additional online rotations which is more expensive.
\RG{Quip sharp was removed, maybe it could still be included in some appendix plots? or we could comment on this only.}
\AG{The QUIP\# numbers with the stable hessians are missing. I will modify the comment for now and add the plot later in the Appendix, need to write an additional script for it.}
\AG{Removed QuIP\# from the plot}
OptRot finds smallest weight incoherence in the out projection and down projection layers, while \OptRotH is the smallest in the other layers.
Both OptRot and \OptRotH improve over the other methods.
The down-projection layer in particular is known to be affected by outliers \citep{dettmers2022gpt3}, here OptRot consistently achieves lower weight incoherence.

\Cref{fig:inc-main} (second row) plots the Hessian incoherence $\mu_H$ as per the GPTQ bound in \Cref{thm:gptq-bound-quip} derived by \citet{chee2023quip}. 
The third row plots $\tr(D)/\tr(H)$ and the fourth row plots the upper bound (see \Cref{eq:smoothbound}) normalized by $\tr(H)$. 
Note that the $\tr(H)$ is rotation invariant, i.e, constant for a layer across all methods.
We observe that the Hessian incoherence, $\tr(D)$ and its upper bound are almost identical across different rotations, except for the identity which increases these quantities.

We also measure the effect of learned rotations on the quantization error in each layer, by plotting the Signal-to-Noise Ratio (SNR) in \Cref{fig:inc-main} (last row).
The SNR is defined in \Cref{app:snr} in the appendix.
A small quantization error implies a high SNR.
When quantizing with GPTQ, OptRot achieves the highest SNR for the out projection and down projection layers, while \OptRotH achieves the highest SNR for the other layer types.
Overall, we note that improving the upper bounds to the layerwise quantization errors by having a lower weight incoherence translates to higher SNR, validating our approach.

Since we observe that except for the identity rotation, all methods achieve almost identical data-dependent metrics (Hessian incoherence, $\tr(D)$ and $\text{UB}$), we ask ourselves: why does optimizing the data-dependent objective in \OptRotH improve downstream performance as we show in the next section?
A potential explanation is that the data-dependent terms in \OptRotH  play the role of layerwise importance scores for the weight incoherence when summing over layers, in \Cref{eq:optrot-obj}. This allows to have a better approximation of the KL divergence compared to OptRot.

\begin{table*}[h!]
    \centering
    \caption{Results for W4A8 quantization, where activations are quantized with RTN and weights with GPTQ.}
    \label{tab:a8w4_performance}
    \begin{tabular}{l@{\hskip 0.7cm}ccccccccc}
        \toprule
        \multirow{2}{*}{\textbf{Method}} & \multicolumn{3}{c}{Llama-3.2-1B} & \multicolumn{3}{c}{Llama-3.2-3B} & \multicolumn{3}{c}{Llama-3.1-8B} \\
        \cmidrule(lr){2-4} \cmidrule(lr){5-7} \cmidrule(lr){8-10}
        & Acc $\uparrow$ & Wiki $\downarrow$ & KL $\downarrow$ & Acc $\uparrow$ & Wiki $\downarrow$ & KL $\downarrow$ & Acc $\uparrow$ & Wiki $\downarrow$ & KL $\downarrow$ \\
        \midrule
        FP16 & $56.77$ & $9.76$ & $0$ & $64.72$ & $7.81$ & $0$ & $70.29$ & $6.24$ & $0$\\
        \midrule
        No Rotation & $48.9$ & $14.49$ & $0.367$ & $54.93$ & $13.03$ & $0.367$ & $66.38$ & $7.65$ & $0.249$ \\
        Quarot & $55.2$ & $10.8$ & $0.137$ & $63.23$ & $8.37$ & $0.099$ & $69.17$ & $6.72$ & $0.094$ \\
        SpinQuant & $55.21$ & $10.75$ & $0.133$ & $\textbf{63.95}$ & $8.34$ & $0.098$ & $69.5$ & $6.72$ & $0.096$ \\
        \midrule
        OptRot (top-50) & $55.1$ & $10.68$ & $0.124$ & $63.28$ & $8.32$ & $\textbf{0.092}$ & $69.4$ & \underline{$6.66$} & \underline{$0.082$} \\
        OptRot (all) & $\textbf{55.55}$ & $10.63$ & $0.126$ & $63.33$ & $\textbf{8.29}$ & \underline{$0.094$} & $69.32$ & $\textbf{6.65}$ & $0.087$\\
        \midrule
        OptRot-v2 & \underline{$55.34$} & \underline{$10.6$} & $0.117$ & $63.64$ & $8.34$ & $0.1$ & $66.42$ & $8.36$ & $0.264$\\
        \OptRotH & $55.05$ & $\textbf{10.58}$ & $\textbf{0.115}$ & $62.45$ & \underline{$8.31$} & $0.87$ & $\textbf{70.17}$ & \underline{$6.66$} & $\textbf{0.081}$\\
        \OptRotH-v2 & $55.31$ & $10.64$ & \underline{$0.116$} & \underline{$63.76$} & $8.33$ & $0.87$ & \underline{$69.73$} & $6.69$ & $\textbf{0.081}$\\ 
        \bottomrule
    \end{tabular}
\end{table*}

\begin{table*}[ht!]
    \centering
    \caption{4-bit weight-only quantization with GPTQ on Qwen3 models.}
    \label{tab:weight-only-qwen3}
    \setlength{\tabcolsep}{4pt} %
    \begin{tabular}{l@{\hskip 0.5cm}ccc@{\hskip 0.5cm}ccc@{\hskip 0.5cm}ccc}
        \toprule
        \multirow{2}{*}{\textbf{Method}} & \multicolumn{3}{c}{Qwen3-1.7B} & \multicolumn{3}{c}{Qwen3-4B} 
        & \multicolumn{3}{c}{Qwen3-8B} \\
        \cmidrule(lr){2-4} \cmidrule(lr){5-7} %
        \cmidrule(lr){8-10}
         & Acc $\uparrow$ & Wiki $\downarrow$ & KL $\downarrow$ & Acc $\uparrow$ & Wiki $\downarrow$ & KL $\downarrow$ 
        & Acc $\uparrow$ & Wiki $\downarrow$ & KL $\downarrow$ \\
        \midrule
        FP16 & $57.09$ & $16.62$ &$0$ & $63.97$  & $13.56$& $0$& $67.4$ & $9.81$& $0$  \\
        \midrule
        No Rotation & $51.49$ & $39.25$ & $0.5$ & $60.79$& $14.43$ & $0.21$ & $65.98$ & $11.62$& $0.197$\\
        QuaRot & $54.75$ & $17.75$& $0.089$ & $62.3$& $14.43$ &$0.064$ & $66.53$ & $10.12$&$0.043$\\
        SpinQuant &$\mathbf{55.1}$ & $17.75$ & $0.091$& $62.06$ & $\mathbf{13.81}$& $0.064$& $66.62$ & $9.93$ & $0.041$\\
        \midrule
        OptRot & $55.03$ & $\mathbf{17.5}$& $\mathbf{0.083}$& $\mathbf{62.9}$& $14.43$& $\mathbf{0.059}$ & $\mathbf{66.74}$& $9.93$& $\mathbf{0.039}$\\
        \bottomrule
    \end{tabular}
\end{table*}

\subsection{End-to-End Results}

\textbf{OptRot improves weight-only quantization.}
Results for weight-only quantization with GPTQ, are shown in Table \ref{tab:weight-only} at $4$-bits.
Along with SpinQuant and QuaRot, we also compare with recent methods OSTQuant and DartQuant.
OSTQuant differs from our setting in ~\Cref{fig:rotations} as it uses an invertible scaling in addition to learnable rotations (shown with $*$).
DartQuant follows the same setup as OptRot but optimizes a different (Whip) loss, which we observe achieves worse final weight incoherence compared to OptRot.  
OptRot consistently outperforms other methods, even with the cheaper top-$50$ version.
This result also holds in the absence of online rotations (see Table \ref{tab:no_r3r4_performance}) and at $3$-bit quantization (see Table \ref{tab:3bit_performance}).

With the underperforming RTN quantization method, SpinQuant outperforms OptRot (see Appendix \ref{sec:rtn-results}), since in this case we can use weight quantization also when learning rotations using SpinQuant.
Note that since the OptRot objective does not contain any data-dependent terms, it is a valid proxy also for the RTN bound in \Cref{thm:rtnbound}.

\textbf{Activation Quantization.} 
We also report results for quantizing both weights and activations. 
The weights are quantized with GPTQ and the activations with RTN.
OptRot matches or outperforms SpinQuant in the A8W4 setting (see Table \ref{tab:a8w4_performance}).
This setting can efficiently serve large models without a significant drop in performance from FP16 \citep{lin2024qserve}.
In contrast, in the A4W4 setting, we see a significant drop in performance compared to the full precision model (as seen in Table \ref{tab:a4w4_performance}).
Here, OptRot performs worse than SpinQuant and often even QuaRot.
We hypothesize that this is due to the activation outliers being dominant compared to the weight outliers and
improving the weight quantization error bound with OptRot worsens activation quantization.
This suggests a trade off between weight and activation outliers.

\textbf{Results on Qwen models.}
We also report results for OptRot on the Qwen3 family of models for weight-only quantization in ~\Cref{tab:weight-only-qwen3}.
The Qwen3 models follow the broad structure of Llama models with the additional RMSnorm after the key and query layers. However, this does not change how the rotations are applied.
Here too, we observe that OptRot always achieves a lower KL divergence in comparison to SpinQuant and QuaRot.

\textbf{Method Selection.}
Based on ~\Cref{tab:weight-only} and ~\Cref{tab:a8w4_performance}, we find that OptRot improves the KL divergence in comparison to other methods.
A lower KL divergence implies better quantization, and can be further improved with a better calibration set.
\OptRotH improves over OptRot, but the improvements are small and come at the cost of optimizing the rotations over the Hessians.
Given the observation that the data-dependent terms only play a role of layerwise importance, one could only compute Hessians once in the beginning and use Hadamard rotations to get the optimization objective for learning rotations making \OptRotH cheaper.
We recommend choosing \OptRotH only if Hessian computation for rotation learning is feasible, but in most cases the simpler data-free OptRot is sufficient to achieve good performance. 
\RG{We can mention that we can improve the calibration data possibly to improve the other metrics (not sure though).}
\AG{Given the observation that the data-dependent terms only plays a role of layerwise importance, one could only compute Hessians once in the beginning and use Hadamard rotations to get the optimization objective for learning rotations.}

\section{Conclusion}
We introduce OptRot and \OptRotH, to learn rotations for post-training weight quantization by minimizing principled, cheaper proxy objectives to the quantization error.
Unlike other methods, learning rotations with OptRot does not require quantizing the model or loading the model on GPU(s) and peforming multiple optimization steps.  
Building on \citet{tseng2024quip}, we leverage the fact that rotations control the weight incoherence and feature correlation, which in turn bound the weight quantization error for GPTQ.
OptRot proposes to optimize these quantities to improve these bounds, which translates to improved downstream performance. 
When including activation quantization, while OptRot is competitive with SpinQuant in the A8W4 regime, it underperforms in the A4W4 one due to a trade-off between weight and activation outliers.

\newpage
\bibliography{icml_references}

\begin{thebibliography}{26}
\providecommand{\natexlab}[1]{#1}
\providecommand{\url}[1]{\texttt{#1}}
\expandafter\ifx\csname urlstyle\endcsname\relax
  \providecommand{\doi}[1]{doi: #1}\else
  \providecommand{\doi}{doi: \begingroup \urlstyle{rm}\Url}\fi

\bibitem[Achiam et~al.(2023)Achiam, Adler, Agarwal, Ahmad, Akkaya, Aleman, Almeida, Altenschmidt, Altman, Anadkat, et~al.]{achiam2023gpt}
Achiam, J., Adler, S., Agarwal, S., Ahmad, L., Akkaya, I., Aleman, F.~L., Almeida, D., Altenschmidt, J., Altman, S., Anadkat, S., et~al.
\newblock Gpt-4 technical report.
\newblock \emph{arXiv preprint arXiv:2303.08774}, 2023.

\bibitem[Akhondzadeh et~al.(2025)Akhondzadeh, Bojchevski, Eleftheriou, and Dazzi]{akhondzadeh2025kurtail}
Akhondzadeh, M.~S., Bojchevski, A., Eleftheriou, E., and Dazzi, M.
\newblock Kurtail: Kurtosis-based llm quantization.
\newblock \emph{arXiv preprint arXiv:2503.01483}, 2025.

\bibitem[Ashkboos et~al.(2024)Ashkboos, Mohtashami, Croci, Li, Cameron, Jaggi, Alistarh, Hoefler, and Hensman]{ashkboos2024quarot}
Ashkboos, S., Mohtashami, A., Croci, M.~L., Li, B., Cameron, P., Jaggi, M., Alistarh, D., Hoefler, T., and Hensman, J.
\newblock Quarot: Outlier-free 4-bit inference in rotated llms.
\newblock \emph{Advances in Neural Information Processing Systems}, 37:\penalty0 100213--100240, 2024.

\bibitem[Bisk et~al.(2020)Bisk, Zellers, Le~Bras, Gao, and Choi]{bisk2020piqa}
Bisk, Y., Zellers, R., Le~Bras, R., Gao, J., and Choi, Y.
\newblock {PIQA}: Reasoning about physical commonsense in natural language.
\newblock In \emph{Proceedings of the AAAI Conference on Artificial Intelligence}, volume~34, pp.\  7432--7439, 2020.

\bibitem[Chee et~al.(2023)Chee, Cai, Kuleshov, and De~Sa]{chee2023quip}
Chee, J., Cai, Y., Kuleshov, V., and De~Sa, C.~M.
\newblock Quip: 2-bit quantization of large language models with guarantees.
\newblock \emph{Advances in Neural Information Processing Systems}, 36:\penalty0 4396--4429, 2023.

\bibitem[Clark et~al.(2018)Clark, Cowhey, Etzioni, Khot, Sabharwal, Schoenick, and Tafjord]{clark2018think}
Clark, P., Cowhey, I., Etzioni, O., Khot, T., Sabharwal, A., Schoenick, C., and Tafjord, O.
\newblock Think you have solved question answering? try {ARC}, the {AI2} reasoning challenge.
\newblock \emph{arXiv preprint arXiv:1803.05457}, 2018.

\bibitem[Dettmers et~al.(2022)Dettmers, Lewis, Belkada, and Zettlemoyer]{dettmers2022gpt3}
Dettmers, T., Lewis, M., Belkada, Y., and Zettlemoyer, L.
\newblock Gpt3. int8 (): 8-bit matrix multiplication for transformers at scale.
\newblock \emph{Advances in neural information processing systems}, 35:\penalty0 30318--30332, 2022.

\bibitem[Frantar et~al.(2022)Frantar, Ashkboos, Hoefler, and Alistarh]{frantar2022gptq}
Frantar, E., Ashkboos, S., Hoefler, T., and Alistarh, D.
\newblock Gptq: Accurate post-training quantization for generative pre-trained transformers.
\newblock \emph{arXiv preprint arXiv:2210.17323}, 2022.

\bibitem[Grattafiori et~al.(2024)Grattafiori, Dubey, Jauhri, Pandey, Kadian, Al-Dahle, , et~al.]{grattafiori2024llama3}
Grattafiori, A., Dubey, A., Jauhri, A., Pandey, A., Kadian, A., Al-Dahle, A., , et~al.
\newblock The llama 3 herd of models.
\newblock \emph{arXiv preprint arXiv:2407.21783}, 2024.

\bibitem[Hu et~al.(2025)Hu, Cheng, Yang, Chen, Xu, Yuan, Zhou, et~al.]{ostquant}
Hu, X., Cheng, Y., Yang, D., Chen, Z., Xu, Z., Yuan, Z., Zhou, S., et~al.
\newblock Ostquant: Refining large language model quantization with orthogonal and scaling transformations for better distribution fitting.
\newblock In \emph{The Thirteenth International Conference on Learning Representations}, 2025.

\bibitem[Li et~al.(2020)Li, Fuxin, and Todorovic]{li2020efficient}
Li, J., Fuxin, L., and Todorovic, S.
\newblock Efficient riemannian optimization on the stiefel manifold via the cayley transform.
\newblock \emph{arXiv preprint arXiv:2002.01113}, 2020.

\bibitem[Lin et~al.(2024)Lin, Tang, Yang, Zhang, Xiao, Gan, and Han]{lin2024qserve}
Lin, Y., Tang, H., Yang, S., Zhang, Z., Xiao, G., Gan, C., and Han, S.
\newblock Qserve: W4a8kv4 quantization and system co-design for efficient llm serving.
\newblock \emph{arXiv preprint arXiv:2405.04532}, 2024.

\bibitem[Liu et~al.(2024)Liu, Zhao, Fedorov, Soran, Choudhary, Krishnamoorthi, Chandra, Tian, and Blankevoort]{liu2024spinquant}
Liu, Z., Zhao, C., Fedorov, I., Soran, B., Choudhary, D., Krishnamoorthi, R., Chandra, V., Tian, Y., and Blankevoort, T.
\newblock Spinquant: Llm quantization with learned rotations.
\newblock \emph{arXiv preprint arXiv:2405.16406}, 2024.

\bibitem[Merity et~al.(2017)Merity, Xiong, Bradbury, and Socher]{merity2017pointer}
Merity, S., Xiong, C., Bradbury, J., and Socher, R.
\newblock Pointer sentinel mixture models.
\newblock In \emph{International Conference on Learning Representations}, 2017.

\bibitem[Paperno et~al.(2016)Paperno, Kruszewski, Lazaridou, Pham, Bernardi, Pezzelle, Baroni, Boleda, and Fern{\'a}ndez]{paperno2016lambada}
Paperno, D., Kruszewski, G., Lazaridou, A., Pham, Q.~N., Bernardi, R., Pezzelle, S., Baroni, M., Boleda, G., and Fern{\'a}ndez, R.
\newblock The {LAMBADA} dataset: Word prediction requiring a broad discourse context.
\newblock In \emph{Proceedings of the 54th Annual Meeting of the Association for Computational Linguistics (Volume 1: Long Papers)}, pp.\  1525--1534, 2016.

\bibitem[Raffel et~al.(2019)Raffel, Shazeer, Roberts, Lee, Narang, Matena, Zhou, Li, and Liu]{c4-dataset}
Raffel, C., Shazeer, N., Roberts, A., Lee, K., Narang, S., Matena, M., Zhou, Y., Li, W., and Liu, P.~J.
\newblock Exploring the limits of transfer learning with a unified text-to-text transformer.
\newblock \emph{arXiv e-prints}, 2019.

\bibitem[Sakaguchi et~al.(2020)Sakaguchi, Le~Bras, Bhagavatula, and Choi]{sakaguchi2020winogrande}
Sakaguchi, K., Le~Bras, R., Bhagavatula, C., and Choi, Y.
\newblock {WinoGrande}: An adversarial winograd schema challenge at scale.
\newblock In \emph{Proceedings of the AAAI Conference on Artificial Intelligence}, volume~34, pp.\  8732--8740, 2020.

\bibitem[Shao et~al.(2025)Shao, Chen, Wang, Yu, Lin, Yao, Wei, and Cheng]{shao2025dartquant}
Shao, Y., Chen, Y., Wang, P., Yu, J., Lin, J., Yao, Y., Wei, Z., and Cheng, J.
\newblock Dartquant: Efficient rotational distribution calibration for {LLM} quantization.
\newblock In \emph{The Thirty-ninth Annual Conference on Neural Information Processing Systems}, 2025.
\newblock URL \url{https://openreview.net/forum?id=LfcfwlLCHM}.

\bibitem[Sun et~al.(2024)Sun, Liu, Bai, Bao, Zhao, Li, Hu, Yu, Hou, Yuan, et~al.]{sun2024flatquant}
Sun, Y., Liu, R., Bai, H., Bao, H., Zhao, K., Li, Y., Hu, J., Yu, X., Hou, L., Yuan, C., et~al.
\newblock Flatquant: Flatness matters for llm quantization.
\newblock \emph{arXiv preprint arXiv:2410.09426}, 2024.

\bibitem[Tseng et~al.(2024)Tseng, Chee, Sun, Kuleshov, and De~Sa]{tseng2024quip}
Tseng, A., Chee, J., Sun, Q., Kuleshov, V., and De~Sa, C.
\newblock Quip\#: Even better llm quantization with hadamard incoherence and lattice codebooks.
\newblock \emph{arXiv preprint arXiv:2402.04396}, 2024.

\bibitem[Xiao et~al.(2023)Xiao, Lin, Seznec, Wu, Demouth, and Han]{xiao2023smoothquant}
Xiao, G., Lin, J., Seznec, M., Wu, H., Demouth, J., and Han, S.
\newblock Smoothquant: Accurate and efficient post-training quantization for large language models.
\newblock In \emph{International conference on machine learning}, pp.\  38087--38099. PMLR, 2023.

\bibitem[Xu et~al.(2025)Xu, Dong, Elachqar, and Shang]{xu2025butterflyquant}
Xu, B., Dong, Z., Elachqar, O., and Shang, Y.
\newblock Butterflyquant: Ultra-low-bit llm quantization through learnable orthogonal butterfly transforms.
\newblock \emph{arXiv preprint arXiv:2509.09679}, 2025.

\bibitem[Xu et~al.(2024)Xu, Lan, Yazar, Webb, Sharify, and Wang]{xu2024scaling}
Xu, Z., Lan, A., Yazar, W., Webb, T., Sharify, S., and Wang, X.
\newblock Scaling laws for post-training quantized large language models.
\newblock \emph{CoRR}, 2024.

\bibitem[Yang et~al.(2025)Yang, Li, Yang, Zhang, Hui, Zheng, Yu, Gao, Huang, Lv, et~al.]{yang2025qwen3}
Yang, A., Li, A., Yang, B., Zhang, B., Hui, B., Zheng, B., Yu, B., Gao, C., Huang, C., Lv, C., et~al.
\newblock Qwen3 technical report.
\newblock \emph{arXiv preprint arXiv:2505.09388}, 2025.

\bibitem[Zellers et~al.(2019)Zellers, Holtzman, Bisk, Farhadi, and Choi]{zellers2019hellaswag}
Zellers, R., Holtzman, A., Bisk, Y., Farhadi, A., and Choi, Y.
\newblock {HellaSwag}: Can a machine really finish your sentence?
\newblock In \emph{Proceedings of the 57th Annual Meeting of the Association for Computational Linguistics}, pp.\  4791--4800, 2019.

\bibitem[Zhang et~al.(2024)Zhang, Wang, Deng, Li, Yang, and Yin]{zhang2024magr}
Zhang, A., Wang, N., Deng, Y., Li, X., Yang, Z., and Yin, P.
\newblock Magr: Weight magnitude reduction for enhancing post-training quantization.
\newblock \emph{Advances in neural information processing systems}, 37:\penalty0 85109--85130, 2024.

\end{thebibliography}
\bibliographystyle{icml2026}

\appendix
\onecolumn
\newpage

\section{Round to Nearest Error Bound}\label{app:RTN}

Round-to-Nearest (RTN) is a simple quantization method which simply rounds every weight element to its nearest discrete value on the quantization grid.
For a bit-width $b$, the weight values must be scaled such that $0 \leq W_{ij} \leq 2^b - 1$ assuming symmetric quantization.

Each element is rescaled as $W_{ij} \rightarrow \frac{2^b-1}{2}\left(\frac{W_{ij}}{w_{\max}} + 1\right)$ with $w_{\max} = \max_{i,j} |W_{i,j}|$ followed by rounding to the nearest integer.
The error achieved by the RTN procedure can be upper bounded by the extreme case where each element is exactly in the middle of a quantization interval and incurs an error of $\frac{\Delta}{2}$ where $\Delta = \frac{2w_{\max}}{(2^b - 1)}$.
The following theorem derives this bound.

\begin{theorem}[Worst Case Error Bound for RTN]
Let $W \in \R^{m \times n}$ be a weight matrix and $H \in \R^{n \times n}$ be a symmetric positive semi-definite (PSD) matrix. Let $\hat{W}$ be the matrix obtained by applying uniform $b$-bit Round-to-Nearest (RTN) quantization to each element of $W$. The quadratic quantization error is bounded as follows:
$$ \tr((\hat{W} - W)H(\hat{W} - W)^\top) \leq \frac{\mu_W^2}{(2^b-1)^2} \lambda_{\max}(H)  \norm{W}^2_F $$
where $\mu_W$ is the weight incoherence and $\lambda_{\max}(H)$ is the maximum eigenvalue of the matrix $H$.
\end{theorem}

\begin{proof}
Let $\eta := \hat{W} - W$ be the quantization error matrix. The objective function is $\tr(\eta H \eta^\top)$.
By the cyclic property of the trace, we can rearrange the terms:
    $$\tr(\eta H \eta^\top) = \tr(H \eta^\top \eta) \,.$$
We apply the Von Neumann’s trace inequality $\tr(AB) \le \sum_{i=1}^{n}\lambda_i(A)\lambda_i(B)$, where $(\lambda_i(A))_{i=1}^n$, $(\lambda_i(B))_{i=1}^n$ are the eigenvalues of $A$ and $B$, which holds if $A, B$ are PSD. In our case $A = H$, $B = \eta^\top \eta$
    \begin{align*}
        \tr(H \eta^\top \eta) &\leq \sum_{i=1}^{n}\lambda_i(H)\lambda_i(\eta^\top\eta)
        \leq \lambda_{\max}(H) \tr(\eta^\top\eta)     
        =  \lambda_{\max}(H)\|\eta\|_F^2  \,.   
    \end{align*}    

Next, we bound the Frobenius norm of the error matrix $\eta$. For a uniform $b$-bit quantization scheme, the width of each quantization interval is $\Delta = \frac{2w_{\max}}{(2^b - 1)}$. For Round-to-Nearest, the maximum error for any single element is half of this interval width:
    $$|\eta_{ij}| \le \frac{\Delta}{2} = \frac{w_{\max}}{2^{b}-1}$$
Using this per-element bound, we can bound the squared Frobenius norm, which is the sum of the squared magnitudes of all elements:
    $$ \|\eta\|_F^2 = \sum_{i=1}^m \sum_{j=1}^n |\eta_{ij}|^2 \le \sum_{i=1}^m \sum_{j=1}^n \left(\frac{w_{\max}}{2^{b}-1}\right)^2 = mn \left(\frac{w_{\max}}{2^{b}-1}\right)^2 $$
Finally, substituting the bound for $\eta$ yields
    $$\tr(\eta H \eta^\top) \le mn \left(\frac{w_{\max}}{2^b - 1}\right)^2 \lambda_{\max}(H)\,.$$
We conclude the proof by using the definition of $\mu_W$.
\end{proof}

\section{Approximating the KL Divergence with the Layerwise Errors}\label{app:approx-kl}
In this section, we describe the steps needed to approximate the KL divergence between quantized and original model as the sum of layerwise reconstruction errors. Such an approximation can be inaccurate but will be cheaper to minimize. 

We begin by recalling our main objective:
\begin{align*}
    \mathbb{E}_z  \mathrm{KL} (p_\theta(\cdot \mid  z) \,\|\ p_{\hat \theta}( \cdot \mid z))\,, \label{eq:kl}
\end{align*}
where $\mathrm{KL}(p \,\|\ q) = \mathbb{E}_{y \sim p} \left[\log \left(p(y)/q(y)\right)\right]
$ and $p_\theta(\cdot \mid  z)$ denotes the output probability distribution of a language model given an input sequence $z$. The original weights are denoted by $\theta$ while the rotated and then quantized weights with $\hat \theta$.

First, we approximate this objective by a second-order Taylor expansion around $\theta$, yielding the quadratic surrogate
\begin{equation}
\label{eq:global-quad}
 \tfrac{1}{2}\,(\theta-\hat\theta)^\top F(\theta)\,(\theta-\hat\theta), \quad F(\theta) := \E_{z, y \sim p_\theta (\,\cdot \mid z)} \left[ \nabla_\theta \log p_\theta(y \mid z) \, \nabla_\theta \log p_\theta(y \mid z)^\top\right]
\end{equation}
where $F$ denotes the Fisher information matrix (FIM), equal to the Hessian of the KL at $\theta$ under standard regularity conditions. Directly solving \eqref{eq:global-quad} is intractable due to the size of $F$ which  for LLMs is greater than a billion squared.

We adopt a block-diagonal approximation across layers, splitting $\theta$ into $\{W^\ell\}_\ell$ where each $W$ correspond to each linear layer that will be quantized, and making a block-diagonal approximation of $F(\theta)$ with each block corresponding to one layer:
\begin{equation}
\label{eq:layer-quad}
\sum_{W, F_W} \tfrac{1}{2}\,\Vec(W-\hat W)^\top F_W\,\Vec(W-\hat W),
\qquad
F_W = \E\!\big[\Vec(\nabla_W \log p_\theta)\,\Vec(\nabla_W \log p_\theta)^\top\big], 
\end{equation}
In language models, the inputs and outputs probabilities are sequences. If we let \(\{x_t\}_{t=1}^T \) be the sequence of inputs to layer $W$, with outputs \(y_t = W x_t\) and output-gradients \(g_t = \frac{d}{dy_t} \log p_\theta\), then
\[
\nabla_W \log p_\theta = \sum_{t=1}^T  g_t x_t^\top,
\qquad
\Vec(\nabla_W \log p_\theta) = \sum_{t=1}^T  (x_t \otimes g_t).
\]
Thus, the fisher information matrix for layer $W$ can be written as
\begin{equation*}
F_W =\E\!\Bigg[\sum_{t,s=1}^T  (x_t \otimes g_t ) ( x_s \otimes g_s)^\top\Bigg]= \E\!\Bigg[\sum_{t,s=1}^T  (g_t g_s^\top) \otimes (x_t x_s^\top)\Bigg],
\end{equation*}
Computing the full fisher for a single layer is still intractable, therefore we use the following kronecker factorization
\[
F_W \approx \,G\otimes H,
\qquad
G=\E\Big[ T\sum_{t=1}^T g_t g_t^\top\Big],\;H=\E\Big[T \sum_{t=1}^T x_t x_t^\top\Big],
\]
Which holds under the assumption that (i) $g_i$ is independent from $x_j$ for any $i,j$ pair (gradients are independent from the input) and (ii) the random variables $(x_i, g_i), \dots (x_t, g_t)$ are i.i.d..

Finally, if we approximate $G$ with the identity matrix, we get to the layerwise reconstruction error since if $G = I$, then 
\begin{equation}
\sum_{W, F_W} \tfrac{1}{2}\,\Vec(W-\hat W)^\top F_W\,\Vec(W-\hat W) = \sum_{W,H} \tr \big( (\hat{W} - W) H (\hat{W} - W)^\top  \big) = \mathcal{D}_{KL}
\end{equation}
In practice, $H$ is approximated for each layer using a calibration set. Note that since we are optimizing over rotations which are sometimes shared across layers, the optimization problem is not separable across layer as it is for the quantization step.

\section{Using Hessian Information for OptRot}
In OptRot we have identified a data-free objective for learning rotations which improves the bound in \Cref{eq:gpt-inc-bound}.
In this section, we explore optimizing additional data-dependent terms in the bound.
Our goal is to assess if better rotations can be learnt by optimizing the data-dependent terms in combination with the data-free one.
The data-free term arises from the weight incoherence while the data-dependent terms originate from the Hessian.
We first show a derivation of the GPTQ objective and related approximations of the bound.
This exercise allows us to  identify data-dependent objectives for learning rotations.

\textbf{GPTQ.} The GPTQ algorithm finds the quantized weight which approximately minimizes $\mathcal{L}$ by exploiting the LDL decomposition of the matrix $H$. In particular, as shown by \cite{tseng2024quip}, it finds $\hat W_{\mathrm{GPTQ}} = Q(W  + (W - \hat W_{\mathrm{GPTQ}}) U)$ where $U$ is the strictly upper triangular matrix from the LDL decomposition of the hessian $H = (U + I)D(U + I)^\top$.
Obtaining a proper worst-case bound for GPTQ is more challenging than for RTN since the corrections  applied before discretization via $Q$ can push some values far outside the quantization grid. 

Despite this, with some changes  to the GPTQ algorithm (not used in practice) 
the worst-case error can be bounded with high probability.
As outlined in the main text, these changes include the use of a constrained LDL decomposition of the Hessian and the use of stochastic rounding.
We term this method in ~\Cref{eq:gptqs} as GPTQS.  Note that in practice however, QuIP (and our methods), uses the true LDL of the Hessian, and clamp the values to restrict them in the quantization range.
\citet{chee2023quip} show that clamping with the true LDL works better in practice compared to the constrained version.
Below we derive the bounds for GPTQS from the main paper, following the theoretical framework of QuIP \citep{chee2023quip}.

We restate the definition for the constrained LDL.
\begin{definition}[Constrained LDL]
    Given a Hessian matrix $H \in \mathbb{R}^{n \times n}$, we define the constrained LDL, $L$, as the solution to the optimization problem,
\begin{align*}
    minimize&: \tr(HL^\top L) \\
    over&: L \text{ unit upper triangular} \\
    subject\text{ } to&: e_i^\top L^\top Le_i \le 1 + c, \forall i \in \{1,...,n\}. 
\end{align*}
\end{definition}

We now restate Lemma 12 from QuIP which will allow us to derive the quantization error bound for GPTQ.
Our goal is to bound the error in terms of the weight incoherence and the  $\tr(D)$ from the LDL of the Hessian.

\begin{lemma}[Lemma 12 from \citet{chee2023quip}]\label{lemma:quip-lemma-12}
Suppose that we quantize the row vector $w \in \mathbb{R}^{1 \times n}$ using $L$ the solution to the optimization problem
\begin{align*}
    minimize&: \tr(HL^\top L) \\
    over&: L \text{ unit upper triangular} \\
    subject\text{ } to&: e_i^\top L^\top Le_i \le 1 + c, \forall i \in \{1,...,n\} 
\end{align*}
and $\hat{w}= \hat Q\big(w-(\hat{w}-w)(L^{-1}-I), s\big)$, where $\hat Q$ denotes stochastic rounding. Then for any $u \in \mathbb{R}^n$ and any $\delta > 0$
\begin{align}
\mathbf{P}\left(|\hat{w}-w|u \ge \norm{Lu} \sqrt{\frac{1}{2}\log\left(\frac{2}{\delta}\right)}\right) \le \delta
\end{align}
\end{lemma}
\begin{proof}
    See Lemma 12 in QuIP \citep{chee2023quip}.
\end{proof}

Based on this Lemma, we bound the quantization error of GPTQ with the following theorem

\begin{theorem}[GPTQ Bound]
Let $H$ be the Hessian, and $L$ be the constrained LDL as defined in Definition \ref{def:constrained-ldl} with $c = 2\left(\log (\frac{4mn}{\delta})\right)^{-1}$.
Then the quantization error or $\hat W_{\mathrm{GPTQS}}$ is bounded with probability $1-\delta$ as
\begin{align}
&\tr((\hat{W}-W)H(\hat{W}-W)^\top ) 
\le \frac{\mu_W^2}{n(2^b-3)^2}\tr(HL^\top L)\norm{W}_F^2 \cdot \frac{1}
    {2}\log\left(\frac{2mn}{\delta}\right)
\end{align}
\end{theorem}

\begin{proof}
First, from Lemma \ref{lemma:quip-lemma-12}, if $Ue_i$ is the $i$-th eigenvector of $H$, with eigenvalue $\lambda_i$
\begin{align}
\mathbf{P}\left(\lambda_i(e_j^\top (\hat{w}-w)Ue_i)^2 \ge \lambda_i\norm{LUe_i}^2 \cdot \frac{1}{2}\log\left(\frac{2}{\delta}\right)\right) \le \delta
\end{align}
By the union bound, 
\begin{align}
\mathbf{P}\left(\exists i,j,  \: \lambda_i(e_j^\top (\hat{w}-w)Ue_i)^2 \ge \lambda_i\norm{LUe_i}^2 \cdot \frac{1}{2}\log\left(\frac{2mn}{\delta}\right)\right) \le \delta
\end{align}
And so
\begin{align}
\mathbf{P}\left(\sum_{i,j}\lambda_i(e_j^\top (\hat{w}-w)Ue_i)^2 \ge \sum_{i,j}\lambda_i\norm{LUe_i}^2 \cdot \frac{1}{2}\log\left(\frac{2mn}{\delta}\right)\right) \le \delta
\end{align}
which simplifies to 
\begin{align}
\mathbf{P}\left(\tr((\hat{w}-w)H(\hat{w}-w)^\top ) \ge m\tr(HL^\top L) \cdot \frac{1}{2}\log\left(\frac{2mn}{\delta}\right)\right) \le \delta
\end{align}

Hence, with probability $1-\delta$, we have
\begin{align}\label{eqn:trace-bound-ldl-constrained}
    \tr((\hat{w}-w)H(\hat{w}-w)^\top ) \le m\tr(HL^\top L) \cdot \frac{1}{2}\log\left(\frac{2mn}{\delta}\right)
\end{align}
Mapping the weights to the quantized grid as $w_{ij} \rightarrow \frac{2^b - 3}{2}\left(\frac{w_{ij}}{\norm{w}_{\infty}} + 1\right) + 1$, gives us the overall bound
\begin{align}
    \tr((\hat{w}-w)H(\hat{w}-w)^\top ) &\le \frac{m}{(2^b-3)^2}\tr(HL^\top L)w_{max}^2 \cdot \frac{1}
    {2}\log\left(\frac{2mn}{\delta}\right) \nonumber\\
\end{align}
In terms of the weight incoherence, the bound can be written as
\begin{align}
    \tr((\hat{w}-w)H(\hat{w}-w)^\top ) &\le \frac{\mu_W^2}{n(2^b-3)^2}\tr(HL^\top L)\norm{W}_F^2 \cdot \frac{1}
    {2}\log\left(\frac{2mn}{\delta}\right) \label{eq:gptq-constrained-bound}
\end{align}

In order to find the value of $c$ for which the correction terms do not exceed $1$, we can write from the previous lemma,

\begin{align}
\mathbf{P}\left(\exists i,j, \: \lambda_i(e_j^\top (\hat{w}-w)(L^{-1}-I)e_i)^2 \ge \sqrt{\frac{c}{2}\log\left(\frac{4mn}{\delta}\right)}\right) \le \frac{\delta}{2}
\end{align}

Setting $c=2\left(\log\left(\frac{4mn}{\delta}\right)\right)^{-1}$ yields
\begin{align}
    \mathbf{P}\left(\exists i,j, \:  \lambda_i(e_j^\top (\hat{w}-w)(L^{-1}-I)e_i)^2 \ge 1 \right) \le \frac{\delta}{2}
\end{align}
And by another union bound, it holds with probability $1-\delta$

\begin{align}
    \max_{i,j} |e_j^\top (\hat w - w)(L^{-1} - I)e_i| \le 1
\end{align}
As long as this inequality holds, the value we pass in to the stochasitc quantizer will be in range, and thus will the output.

\end{proof}
Constraining $L$, as per Lemma \ref{lemma:quip-lemma-12} ensures that the corrective term $(W - \hat{W}_{\text{GPTQS}})(L^{-1} - I)$, does not exceed $1$.

To prove the general bound on the quantization error in ~\Cref{thm:gptq-bound-quip} as done by \citet{chee2023quip}, we restate and use Lemma 11 from \citet{chee2023quip}.

\begin{lemma}[Lemma 11 from \citet{chee2023quip}]\label{lemma:quip-lemma-11}
Supponse that for positive definite $\mu_H$-incoherent matrix $H \in \mathbb{R}^{n \times n}$ and scalar $c > 0$, $L$ is the solution to the optimization problem
    \begin{align*}
    minimize&: \tr(HL^\top L) \\
    over&: L \text{ unit upper triangular} \\
    subject\text{ } to&: e_i^\top L^\top Le_i \le 1 + c, \forall i \in \{1,...,n\} 
\end{align*}
Then the solution satisfies 
\begin{align*}
    \tr(HL^\top L) \le \frac{\mu_H^2}{n \cdot \min (1, c)} \tr\big(H^{1/2}\big)^2
\end{align*}
\end{lemma}

\begin{theorem}[Theorem 14 in \citet{chee2023quip}]
Let $H$ be the Hessian, and $L$ be the constrained LDL as defined in Definition \ref{def:constrained-ldl} with $c = 2\left(\log (\frac{4mn}{\delta})\right)^{-1}$.
Then the quantization error or $\hat W_{\mathrm{GPTQS}}$ is bounded with probability at least $1-\delta$ as:
    \begin{align}
     &\mathcal{L}(\hat W_{\mathrm{GPTQS}}, H) \nonumber\\
     &%
     \leq \frac{\mu_H^2 \mu_W^2}{n^2(2^b-3)^2} \tr \big(H^{1/2}\big)^2  \norm{W}^2_F  \log \Big(\frac{4mn}{\delta}\Big)^2 \,. 
\end{align}
\end{theorem}

\begin{proof}
    From ~\Cref{thm:gptq-bound-trace} it follows that, 
    \begin{align}
    \mathbf{P}\left(\tr((\hat{w}-w)H(\hat{w}-w)^\top ) \ge m\tr(HL^\top L) \cdot \frac{1}{2}\log\left(\frac{2mn}{\delta}\right)\right) \le \delta
    \end{align}
    applying Lemma \ref{lemma:quip-lemma-11},
    \begin{align}
    \mathbf{P}\left(\tr((\hat{w}-w)H(\hat{w}-w)^\top ) \ge \frac{\mu_H^2 m}{n \cdot \min (1, c)} \tr\big(H^{1/2}\big)^2 \cdot \frac{1}{2}\log\left(\frac{2mn}{\delta}\right)\right) \le \delta
    \end{align}
    and now substituting $\delta \rightarrow \delta / 2$,
    \begin{align}
    \mathbf{P}\left(\tr((\hat{w}-w)H(\hat{w}-w)^\top ) \ge \frac{\mu_H^2 m}{2n \cdot \min (1, c)} \tr\big(H^{1/2}\big)^2 \cdot \log\left(\frac{4mn}{\delta}\right)\right) \le \frac{\delta}{2}
    \end{align}

    It also follows from ~\Cref{thm:gptq-bound-trace}, for $c=2\left(\log\left(\frac{4mn}{\delta}\right)\right)^{-1}$
    \begin{align}
        \mathbf{P} \left(\exists i,j, \, |e_j^\top (\hat w - w)(L^{-1}-I)e_i| \ge 1\right) \le \frac{\delta}{2}
    \end{align}
    By another union bound, it holds with probability $1-\delta$
    \begin{align}
       \tr((\hat{w}-w)H(\hat{w}-w)^\top ) 
     \leq \frac{\mu_H^2 m}{4n} \tr \big(H^{1/2}\big)^2  \log \Big(\frac{4mn}{\delta}\Big)^2 \,. 
    \end{align}
    and 
    \begin{align}
        \max_{i,j} |e_j^\top (\hat w - w)(L^{-1}-I)e_i| \le 1
    \end{align}
    As the above inequality holds, the value of the corrected weights does not exceed the quantization range with high probability.

    Now, the final bound is obtained by mapping the weights to the quantization grid for bit-width $b$, $w_{ij} \rightarrow \frac{2^b - 3}{2}\left(\frac{w_{ij}}{\norm{w}_{\infty}} + 1\right) + 1$
    
    \begin{align}
     \tr((\hat{w}-w)H(\hat{w}-w)^\top ) 
     \leq \frac{\mu_H^2 \mu_W^2}{n^2(2^b-3)^2} \tr \big(H^{1/2}\big)^2  \norm{W}^2_F  \log \Big(\frac{4mn}{\delta}\Big)^2 \,. 
\end{align}
This is what we wanted to show.
\end{proof}

\subsection{Difference Between True and Constrained LDL in Practice.}
In practice we do not use the constrained LDL, which is used in the proofs. 
Instead, we use the true LDL to quantize with GPTQ.
We compare the value of $c$ from the theory to the one observed in practice and find that there is a significant gap between them as shown in \Cref{fig:correction-bound}.
Hence, we observe that approximating the constrained LDL with the true LDL does not satisfy the bound.
We plot the value of the largest element of the LDL and compare it with the bound $1+c$.

\begin{figure}[ht!]
    \centering
    \includegraphics[width=\linewidth]{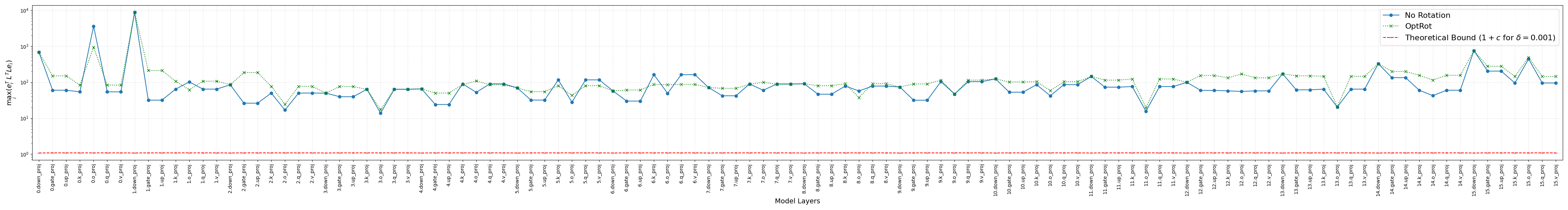}
    \caption{Comparing $\max_i e_i^\top L^\top Le_i$ for the true LDL versus the constrained LDL on a Llama-3.2-1B model.}
    \label{fig:correction-bound}
\end{figure}
Although the gap between our assumed and observed values of $\max_i e_i^\top  L^\top L e_i$ is large, it does not have a significant effect in practice as corrective terms are clamped.
We ignore the effects of clamping on the bound for brevity.

\section{Alternative Bounds for the (Constrained) LDL}

We now present a bound on $\tr(HL^\top L)$ where $L$ is from the constrained LDL definition in \Cref{def:constrained-ldl}.

\begin{lemma}[Constrained LDL Bounds]\label{thm:trace-bound}
Let $H \in \mathbb{R}^{n \times n}$ be a positive semidefinite (PSD) matrix. Let $L$ be the solution to the Constrained LDL problem with parameter $c > 0$ in Definition~\ref{def:constrained-ldl}. Let $\Hoff^2 = \sum_{i \neq j} |H_{ij}|^2$.

Defining the scalar $\alpha = \min(1, \sqrt{c})$, we have that
\[
\tr(H L^\top  L) \le \tr(H) - (2\alpha - \alpha^2) \frac{\Hoff^2}{2\tr(H)} \le \left[ 1 + (1-\alpha)^2 \right] \left( \tr(H) - \frac{\Hoff^2}{2\tr(H)} \right)  \le 2 \left( \tr(H) - \frac{\Hoff^2}{2\tr(H)} \right)
\]
In particular, if $c \ge 1$, then $\tr(H L^\top  L) \le \tr(H) - \frac{\Hoff^2}{2\tr(H)}$ and the $1/2$ factor is sharp.
\end{lemma}

\begin{proof}
We make a constructive upper bound using the candidate $L = I - \frac{\alpha}{\tr(H)} H_L$ where $H_L \in \R^{n \times n}$ contains only the strictly lower triangular part of $H$. To minimize the objective while satisfying the constraint, we select $\alpha = \min(1, \sqrt{c})$.

\paragraph{Part 1: Constraint Verification.}
The constraint requires $(L^\top  L)_{ii} \le 1 + c$. We calculate the diagonal entries of $G = L^\top  L$:
\begin{align*}
G_{ii} = \|L_{:,i}\|_2^2 = 1 + \frac{\alpha^2}{\tr(H)^2} \| (H_L)_{:, i} \|^2 \le 1 + \alpha^2 \frac{\Hoff^2}{\tr(H)^2} \le 1 + \alpha^2
\end{align*}
For the last inequality, we rely on the fundamental property of PSD matrices that for any pair of indices $i, j$ the $2 \times 2$ principal minor is non-negative: $H_{ii}H_{jj} - H_{ij}^2 \ge 0$, which implies $H_{ij}^2 \le H_{ii}H_{jj}$. Summing over all off-diagonal entries:
\[
\Hoff^2 = \sum_{i \neq j} H_{ij}^2 \le \sum_{i \neq j} H_{ii}H_{jj} \le \sum_{i,j} H_{ii}H_{jj} = \left(\sum_i H_{ii}\right)^2 = \tr(H)^2\,.
\] 
Since $\alpha = \min(1, \sqrt{c})$, we have $1+\alpha^2 \le 1+c$ for both $c \ge 1$ and $c < 1$, ensuring feasibility.

\paragraph{Part 2: Derivation of the Bound Chain.}
We expand the objective $\tr(H L^\top  L)$ by substituting the candidate $L = I - \frac{\alpha}{\tau} H_L$ (using $\tau = \tr(H)$ for brevity). First, we expand the product $L^\top  L$:
\[
L^\top  L = \left( I - \frac{\alpha}{\tau} H_L^\top  \right) \left( I - \frac{\alpha}{\tau} H_L \right) = I - \frac{\alpha}{\tau}(H_L + H_L^\top ) + \frac{\alpha^2}{\tau^2} H_L^\top  H_L
\]
Substituting this back into the trace objective, we get three distinct terms:
\begin{align*}
\tr(H L^\top  L) &= \tr\left( H \left[ I - \frac{\alpha}{\tau}(H_L + H_L^\top ) + \frac{\alpha^2}{\tau^2} H_L^\top  H_L \right] \right) 
= \tr(H) - \frac{\alpha}{\tau} \underbrace{\tr(H (H_L + H_L^\top ))}_{\text{Linear Term}} + \frac{\alpha^2}{\tau^2} \underbrace{\tr(H H_L^\top  H_L)}_{\text{Quadratic Term}}
\end{align*}

\textbf{Linear Term:} Since $H$ is symmetric, the matrix $H_L + H_L^\top $ is exactly equal to $H$ with its diagonal elements set to zero. Therefore, the inner product $\tr(H(H_L + H_L^\top ))$ sums the squares of all off-diagonal elements, yielding $\sum_{i \neq j} |H_{ij}|^2 = \Hoff^2$.

\textbf{Quadratic Term:} We use the inequality $\tr(AB) \le \lambda_{\max}(A)\tr(B)$, which holds for any two PSD matrices $A, B$. Setting $A=H$ and $B=H_L^\top  H_L$, we obtain $\tr(H H_L^\top  H_L) \le \lambda_{\max}(H) \tr(H_L^\top  H_L)$. Since $\lambda_{\max}(H) \le \tr(H)$ and $\tr(H_L^\top  H_L) = \|H_L\|_F^2 = \frac{1}{2}\Hoff^2$, this term is bounded by $\tr(H) \frac{\Hoff^2}{2}$.

Plugging these results back into the expansion, and defining $\Delta = \frac{\Hoff^2}{2\tr(H)}$ and $\UB = \tr(H) - \Delta$:
\begin{align*}
\tr(H L^\top  L) &\le \tr(H) - \frac{\alpha}{\tr(H)} \Hoff^2 + \frac{\alpha^2}{\tr(H)^2}  \tr(H) \frac{\Hoff^2}{2}  \\
&= \tr(H) - 2\alpha \frac{\Hoff^2}{2\tr(H)}  + \alpha^2  \frac{\Hoff^2}{2\tr(H)} \\
&= \tr(H) - (2\alpha - \alpha^2) \Delta \quad \text{(\textbf{Additive Bound})}
\end{align*}

Finally, to reach the multiplicative form, we substitute $\tr(H) = \UB + \Delta$ and rearrange:
\begin{align*}
\tr(H L^\top  L) &\le (\UB + \Delta) - (2\alpha - \alpha^2) \Delta \\
&= \UB + (1 - 2\alpha + \alpha^2) \Delta \\
&= \UB + (1-\alpha)^2 \Delta \\
&\le \UB + (1-\alpha)^2 \UB = \left[ 1 + (1-\alpha)^2 \right] \UB
\end{align*}
The last inequality uses the fact that $\Delta \le \UB$. To see why this is true, observe that $\UB = \tr(H) - \Delta$. Thus, the condition $\Delta \le \UB$ is equivalent to $2\Delta \le \tr(H)$. Substituting $\Delta = \frac{\Hoff^2}{2\tr(H)}$, this becomes $\frac{\Hoff^2}{\tr(H)} \le \tr(H)$, or $\Hoff^2 \le \tr(H)^2$. As shown in Part 1, this inequality holds for all PSD matrices.

\textbf{Sharpness.} Consider the case when $c = \infty$ so that $\tr(D) = \tr(HL^\top L)$ with $D$ from the true LDL decomposition. in that case the bound is $\tr(D) \leq \tr(H) - \frac{\norm{H}_{\mathrm{off}}^2}{2\tr(H)}$, which rearranged becomes $2 \geq \frac{\norm{H}_{\mathrm{off}}^2}{\tr(H)(\tr(H)-\tr(D))}$. Set
\[
M=\begin{pmatrix}0&\varepsilon\\[4pt]0&1\end{pmatrix},\qquad \varepsilon>0,
\]
one has $H=MM^\top =\begin{pmatrix}\varepsilon^2&\varepsilon\\[2pt]\varepsilon&1\end{pmatrix}$,
$\tr(H)=1+\varepsilon^2$, $\tr(D)=1$, and $\sum_{i\ne j}H_{ij}^2=2\varepsilon^2$. Hence
\[
\frac{\sum_{i\ne j}H_{ij}^2}{\tr(H)(\tr(H)-\tr(D))}=\frac{2}{1+\varepsilon^2}\xrightarrow{\varepsilon\to0}2,
\]
so the factor $1/2$ in the bound cannot be improved.
\end{proof}

\subsection{Empirical Comparison Between LDL Bounds}\label{app:bound_compare}

\begin{figure*}[t!]
    \centering
    
    \includegraphics[width=\linewidth]{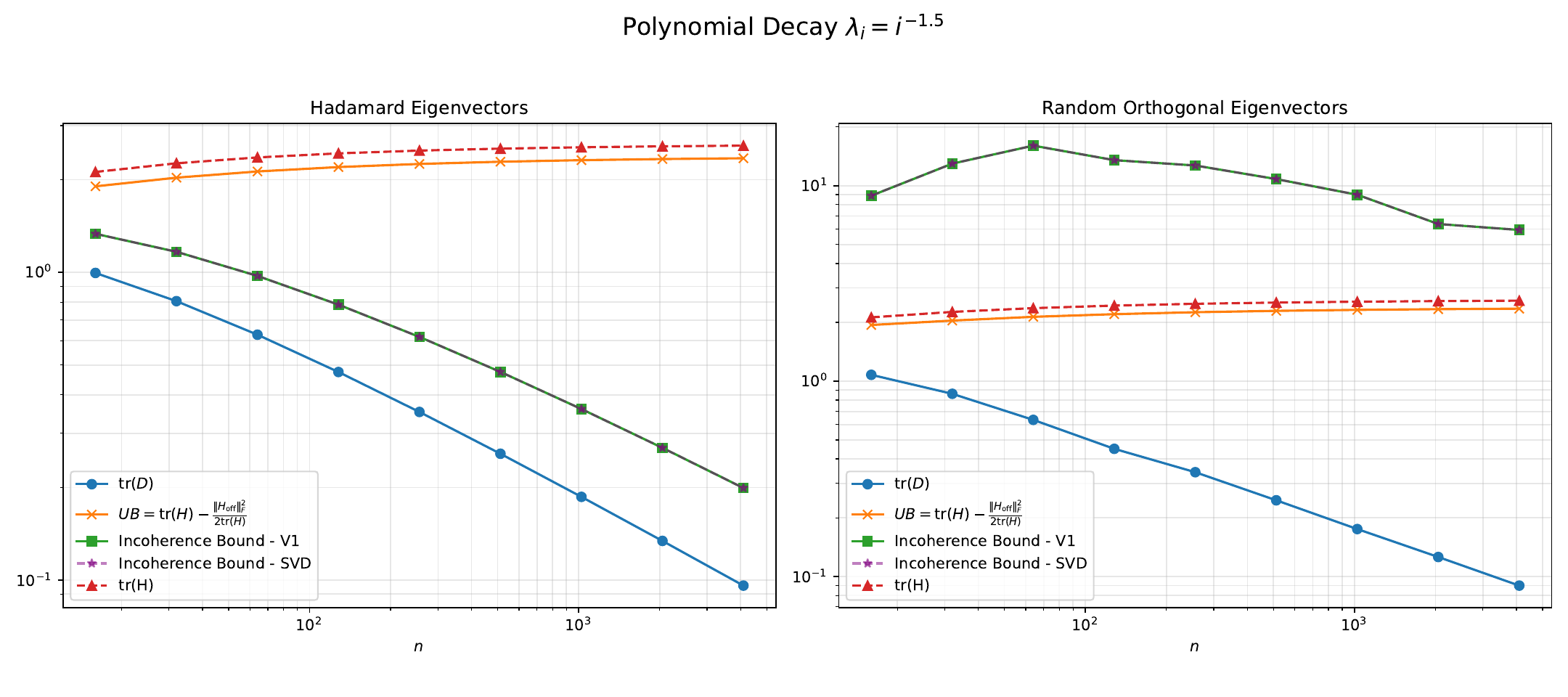}
    \includegraphics[width=\linewidth]{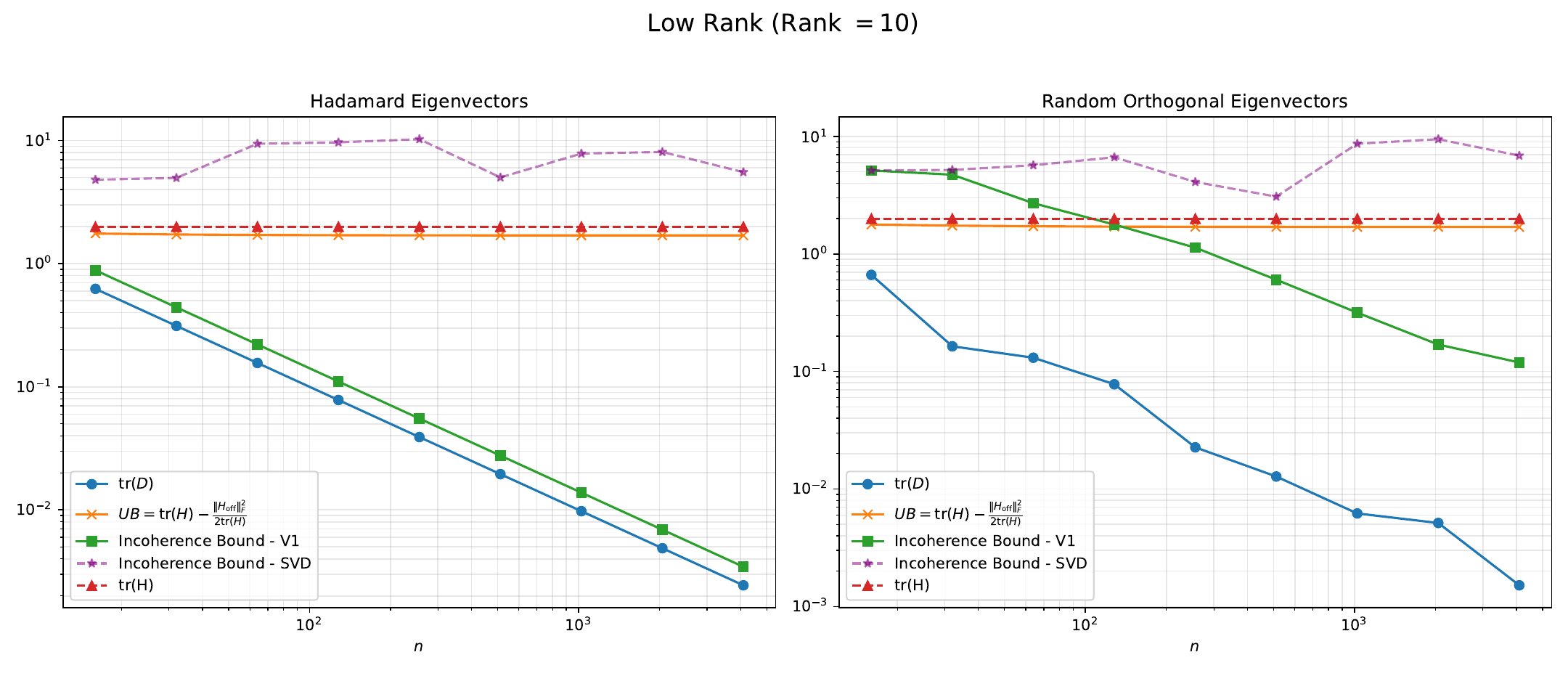}
    \caption{ Comparison among different bounds for $\tr(D)$ when varying the dimension of $H$ ($n \times n$). Incoherence Bound is $\mu_H^2\tr(H^{1/2})^2/n$ with $\mu_H = \max_{i,j}|Q_{i,j}| \sqrt{n}$ and $Q$ is an eigenvector matrix of $H$ which can be the one used to construct $H$ (V1) or the one output of the SVD  scipy function (SVD). There is ambiguity whenever 2 or more eigenvalues coincide and the best choice is $Q$ that minimizes $\max_{i,j}|Q_{i,j}|$. 
    }
    \label{fig:bounds_compare}
\end{figure*}

We compare our proposed bound for the constrained LDL against the bound introduced by \cite{chee2023quip}, which relies on the Hessian incoherence $\mu_H$. We study both bounds in the unconstrained case ($c=\infty$). We construct $n \times n$ matrices of the form $H= Q \mathrm{diag}(\lambda) Q^\top$, where the eigenvector matrix $Q$ is chosen to be either a Hadamard matrix or a random orthogonal matrix. We evaluate two eigenvalue spectra: (1) polynomial scaling where $\lambda_i = i^{1.5}$, and (2) a low-rank setting where $\lambda_i = i^{1.5}$ for $i \leq 10$ and $\lambda_i = 0$ for $i > 10$.

In the low-rank setting, the incoherence bound is ambiguous due to the non-uniqueness of eigenvectors associated with zero eigenvalues; ideally, one would select the basis that minimizes Hessian incoherence. To address this, we compare two versions of the bound: one using the ground-truth $Q$ (optimal for the Hadamard case) and another using eigenvectors computed via standard SVD (representing a practical scenario).

The results are reported in \Cref{fig:bounds_compare}, which also includes the naive bound $\tr(H)$. While the incoherence bound closely approximates $\tr(D)$ for Hadamard eigenvectors, it often exceeds both our bound ($\UB$) and $\tr(H)$ when random eigenvectors are used. Furthermore, in the low-rank case, computing eigenvectors via SVD significantly degrades the incoherence bound. In contrast, our bound remains consistently tighter than $\tr(H)$ and exhibits robust behavior across all configurations, despite being a looser estimate of $\tr(D)$. Note that if $H$ is diagonal $\UB = \tr(H) = \tr(D)$.

\section{Experimental Setup}
\textbf{OptRot.} We learn the rotations by optimizing the proxy objectives described in \Cref{sec:optrot} with Cayley SGD \cite{li2020efficient} (as in SpinQuant) with a learning rate of $1$ for $1000$ steps.
For the cheaper top-$50$ version of OptRot, we learn the rotations by optimizing only over the $50$ weight matrices with the largest loss.

\textbf{SpinQuant/QuaRot.} We use the implementation at \href{https://github.com/facebookresearch/SpinQuant}{https://github.com/facebookresearch/SpinQuant} with default parameters. 
In the case of SpinQuant, to learn rotations, we use $800$ samples of length $2048$  from C4, and we only use activation quantization to $8$-bit (as recommended when using GPTQ) if not otherwise specified. 
Activation quantization to $4$-bits did not improve the result of SpinQuant.
Rotations are optimized with $800$ steps of Cayley SGD with learning rate $1.5$ and batch size $1$. 
To implement QuaRot, we use the same implementation (as SpinQuant) with initial rotations. Hence, we do not include the online Hadamard rotations before the out-projection, which would add inference cost and are less relevant for weight-only quantization. 

\textbf{Quantization.} After rotating the model, we quantize it using GPTQ if not otherwise specified. For GPTQ, we use $512$ samples of length $256$ from the C4 dataset as the calibration set and we set the group-size parameter to $256$ and the damping parameter to $0.01$.

\textbf{Choosing the LR for data-dependent objectives}.
For optimizing with \OptRotH, we first scale the initial loss value to match that of OptRot since the values of the loss for \OptRotH can be very large and lead to numerical issues.
Once scaled, we perform a sweep over learning rates on a log-scale ranging from $1$ to $1e7$ and find that $1e5$ is the best one, which we choose for all \OptRotH experiments.

\textbf{Hessian computation.}
The metrics reported in \Cref{fig:inc-main} and the corresponding figures in the appendix were computed by first evaluating the Hessians over a calibration set and then rotating these Hessians with the learnt rotations of the respective methods.
This method resulted in stable metrics in contrast to measuring them on the fly using the Hessian for GPTQ, which used a smaller calibration set.

\textbf{SNR}.
\label{app:snr}
We also report the Signal-to-Noise Ratio (SNR) in dB to measure the quality of the quantization with GPTQ.
The SNR is computed for GPTQ as:
\begin{align}\label{eq:snr}
    \text{SNR} := 10 * \log_{10} \left( \frac{\E\|Wx\|^2}{\E\|(\hat{W} - W)x\|^2} \right) = 10 * \log_{10} \left( \frac{\tr(W H W^\top)}{\tr((W- \hat W)  H(W- \hat W)^\top)} \right),
\end{align}
where $H = \E [xx^\top]$ and the Expectation is computed over the calibration set. Note that any upper bound on the incoherence can be used to lower bound the SNR.

A high SNR means small quantization error.
According to the bound in Equation \ref{eq:gpt-inc-bound} lowering the weight incoherence should improve the error and increase the SNR.
Our results confirm this observation, the lower weight incoherence achieved by OptRot translates to a higher SNR for weight quantization.

\section{Learning SpinQuant Rotations}
By default, SpinQuant learns rotations by quantizing only the activations to $8$-bits using RTN and a straight-through estimator to backpropagate the gradients to minimize the KL divergence.
Since our focus is on weight quantization, while SpinQuant focuses more on activation quantizaton, we measure  the effect of learning the SpinQuant rotations by quantizing only the weights \textbf{(SpinQuant W4)} or activations to $4$-bits \textbf{(SpinQuant A4)} on Wikitext perplexity.

As we can see in \Cref{tab:spinquant-rot-comparison}, there are not significant improvements for both the A4 and W4 variants when quantizing weights with GPTQ.

Since SpinQuant uses RTN to learn the rotations,
choosing weight or activation quantization has an effect on the downstream performance when quantizing weights with RTN (\Cref{,tab:spinquant-rot-comparison-rtn}), i.e., when the quantization method is aligned for both rotation learning and quantization.
We observe that when the rotations are learnt with weight quantization, weight-only quantization with SpinQuant performs significantly better than the baseline which learns rotations with activation quantization.
This version of SpinQuant (W4) also outperforms OptRot (\Cref{tab:rtn}). See \Cref{sec:rtn-results} for an in-depth discussion.

\begin{table}[ht!]
\centering
\caption{WikiText perplexity for learning SpinQuant rotations with weight or activation quantization followed by weight-only quantization to $4$-bits with GPTQ.}
\label{tab:spinquant-rot-comparison}
\begin{tabular}{lccc}
\toprule
\textbf{Method} & Llama-3.2-1B & Llama-3.2-3B & Llama-3.1-8B \\
\midrule
SpinQuant (Baseline) & $10.73$ & $8.37$  & $6.71$ \\
SpinQuant (A4)      & $10.73$ & $8.36$ & $6.7$ \\
SpinQuant (W4)      & $10.79$ & $8.35$ & $6.69$ \\
\bottomrule
\end{tabular}
\end{table}

\begin{table}[ht!]
\centering
\caption{WikiText perplexity for learning SpinQuant rotations with weight or activation quantization followed by weight-only quantization to $4$-bits with RTN.}
\label{tab:spinquant-rot-comparison-rtn}
\begin{tabular}{lccc}
\toprule
\textbf{Method} & Llama-3.2-1B & Llama-3.2-3B & Llama-3.1-8B \\
\midrule
SpinQuant (Baseline) & $13.56$ & $10.03$  & $7.57$ \\
SpinQuant (W4)      & $11.75$ & $8.97$ & $7.14$ \\
\bottomrule
\end{tabular}
\end{table}

\section{RTN Results}
\label{sec:rtn-results}
We report here additional results concerning weight quantization using RTN. In particular we report weight incoherence in \Cref{fig:mu-rtn}, SNR in \Cref{fig:hessian-snr-rtn} and end-to-end evaluations in \Cref{tab:rtn}.

Learning rotations with OptRot improves the error bound with Round-to-Nearest shown in Equation \ref{eq:rtn-inc-bound}.
In case of RTN, \Cref{eq:rtn-inc-bound} shows that the only rotation invariant part in the bound is weight incoherence and
\Cref{fig:mu-rtn} shows that OptRot achieves the lowest weight incoherence.
This translates to lower error (higher SNR) as shown in \Cref{fig:hessian-snr-rtn}.
\Cref{tab:rtn} reports results for weight-only quantization with RTN, where OptRot outperforms QuaRot (Hadamard rotations).

\textbf{SpinQuant W4}.
\Cref{tab:rtn} also shows that SpinQuant (W4), which learns rotations by quantizing only the weights in the forward pass to 4 bits using RTN, performs significantly better than both SpinQuant and OptRot on downstream tasks. However, the weight incoherence (\Cref{fig:mu-rtn}) and more importantly the SNR (\Cref{fig:hessian-snr-rtn}) plots show that OptRot usually achieves a lower error bound (\Eqref{eq:rtn-inc-bound}) and layerwise objective (\Eqref{eq:trace-obj}) than SpinQuant.
This difference can be explained by the fact that SpinQuant (W4)  optimizes an end-to-end quantization-aware objective: the cross-entropy of the quantized model. Thus, SpinQuant, differently from OptRot, can exploit interactions between weights: the rotations can be learnt in a manner where certain weights compensate for quantization errors in other weights, both within and between different weight matrices.

This effect benefits downstream quantization when the quantization method used for rotation learning and quantization are the same (RTN in this case).
In contrast, while using RTN for rotation learning and quantizing with GPTQ, the performance is either similar to SpinQuant (A4) (see \Cref{tab:spinquant-rot-comparison}) or worse \citep[Table 3]{liu2024spinquant}.
These results suggest that for SpinQuant, aligning the quantization method for both rotation learning and quantization yields the best downstream results.
However, performing this alignment for GPTQ is not straightforward: 
for each step of rotation learning we would have to 
quantize each row of each weight matrix sequentially and update the Hessian matrix, making rotation learning more costly.
Instead, our approach provides a way to more efficiently improve the performance of GPTQ by minimizing an upper bound to the layerwise objective.
We note that GPTQ consistently outperforms RTN, and OptRot + GPTQ achieves the best overall results for weight-only quantization.

\section{Complete Plots}
We report weight incoherence $\mu_W$, Hessian incoherence $\mu_H$, $\tr(D)$, its upper bound and the SNR after quantizing with GPTQ, for the 1B, 3B and 8B models across all layers in Figures \ref{fig:incoherence-all-1B} \ref{fig:incoherence-all-3B} and \ref{fig:incoherence-all-8B}.
OptRot almost always finds the lowest weight incoherence.

\section{Weight Quantization}
In addtion to Table \ref{fig:inc-main} in the main paper, we report additional results for weight-only quantization.

\textbf{Without online rotations.} Table \ref{tab:no_r3r4_performance} reports the performance of OptRot in the absence of the online rotations $R_3, R_4$.
Even without these online rotations, OptRot outperforms SpinQuant and QuaRot.

\textbf{3-bit quantization.} Table \ref{tab:3bit_performance} reports performance of OptRot for $3$-bit weight-only quantization.
OptRot achieves a larger gain over SpinQuant in this regime, further reducing the gap between the $3$-bit model and the full-precision baseline.

\section{Activation Quantization}
We show that OptRot can provably improve the error for weight-only quantization.
While OptRot does not directly optimize for activation quantization, we empirically validate its performance.
We quantize the weights with GPTQ and the activations with Round-to-Nearest (RTN).
We hypothesize that a lower Hessian incoherence benefits activation quantization.
As shown in Figure \ref{fig:inc-main}, OptRot improves the weight incoherence at the cost of increasing the Hessian incoherence, yet still improving the overall weight quantization bound.
Hence, we would expect OptRot to not perform as well as SpinQuant for activation quantization.

\textbf{A8W4}. In this regime, OptRot is still competetive with SpinQuant with lower perplexity across all three models, although the gain is smaller.
Quantizing the activations to 8-bits still achieves an acceptable overall performance for serving models \citep{lin2024qserve}.
Results are reported in Table \ref{tab:a8w4_performance}.

\textbf{A4W4}.
Quantizing activations to 4-bits is a hard problem \cite{lin2024qserve} and the overall model performance drops significantly in this regime.
Here, we observe that OptRot performs worse than SpinQuant. 
Results are reported in Table \ref{tab:a4w4_performance}.

\begin{table}[h!]
    \centering
    \caption{Results without online ($R_3, R_4$) rotations for weight-only quantization at 4-bits with GPTQ.}
    \label{tab:no_r3r4_performance}
    \begin{tabular}{l@{\hskip 0.7cm}ccccccccc}
        \toprule
        \multirow{2}{*}{\textbf{Method}} & \multicolumn{3}{c}{Llama-3.2-1B} & \multicolumn{3}{c}{Llama-3.2-3B} & \multicolumn{3}{c}{Llama-3.1-8B} \\
        \cmidrule(lr){2-4} \cmidrule(lr){5-7} \cmidrule(lr){8-10}
        & Acc $\uparrow$ & Wiki $\downarrow$ & KL $\downarrow$ & Acc $\uparrow$ & Wiki $\downarrow$ & KL $\downarrow$ & Acc $\uparrow$ & Wiki $\downarrow$ & KL $\downarrow$ \\
        \midrule
        FP16 & $56.77$ & $9.76$& $0$ & $64.72$& $7.81$& $0$& $70.29$& $6.24$ & $0$\\
        \midrule
        No Rotation & $49.43$ & $13.86$ & $0.362$ & $56.28$ & $11.78$ & $0.355$ & $67.39$ & $7.47$ & $0.233$ \\
        Quarot & $\textbf{54.6}$ & $11.27$ & $0.208$ & $62.52$ & $8.64$ & $0.183$ & $68.7$ & $6.81$ & $0.154$ \\
        SpinQuant & $54.19$ & $11.21$ & $0.211$ & $62.3$ & $8.69$ & $0.181$ & $68.66$ & $6.79$ & $0.152$ \\
        \midrule
        OptRot (top-50) & $54.4$ & $11.06$ & $0.189$ & $\textbf{63.3}$ & $8.51$ & $\textbf{0.153}$ & $\textbf{69.28}$ & $6.74$ & $\textbf{0.133}$ \\
        OptRot & $54.42$ & $\textbf{10.98}$ & $\textbf{0.185}$ & $62.29$ & $\textbf{8.48}$ & $0.163$ & $69.1$ & $\textbf{6.71}$ & $\textbf{0.133}$\\
        \bottomrule
    \end{tabular}
\end{table}

\begin{table}[h!]
    \centering
    \caption{Results for A4W4 quantization, where activations are quantized with RTN and weights with GPTQ.}
    \label{tab:a4w4_performance}
    \begin{tabular}{l@{\hskip 0.7cm}ccccccccc}
        \toprule
        \multirow{2}{*}{\textbf{Method}} & \multicolumn{3}{c}{Llama-3.2-1B} & \multicolumn{3}{c}{Llama-3.2-3B} & \multicolumn{3}{c}{Llama-3.1-8B} \\
        \cmidrule(lr){2-4} \cmidrule(lr){5-7} \cmidrule(lr){8-10}
        & Acc $\uparrow$ & Wiki $\downarrow$ & KL $\downarrow$ & Acc $\uparrow$ & Wiki $\downarrow$ & KL $\downarrow$ & Acc $\uparrow$ & Wiki $\downarrow$ & KL $\downarrow$ \\
        \midrule
        FP16 & $56.77$ & $9.76$ & $0$ & $64.72$ & $7.81$ & $0$ & $70.29$ & $6.24$ & $0$\\
        \midrule
        No Rotation & $32.22$ & $204$ & $2.88$ & $34.67$ & $367$ & $2.01$ & $37.7$ & $75.68$ & $2.07$ \\
        Quarot & $49.85$ & \underline{$13.67$} & $\textbf{0.392}$ & $57.64$ & \underline{$9.96$} & $\textbf{0.297}$ & $65.05$ & \underline{$7.81$} & \underline{$0.275$} \\
        SpinQuant & $49.63$ & $\textbf{13.56}$ & \underline{$0.393$} & $\textbf{58.69}$ & $\textbf{9.88}$ & \underline{$0.308$} & $\textbf{65.7}$ & $\textbf{7.8}$ & $\textbf{0.273}$ \\
        \midrule
        OptRot (top-50) & $\textbf{50.47}$ & $14.05$ & $0.429$ & $56.92$ & $10.97$ & $0.417$ & $64.05$ & $8.29$ & $0.325$ \\
        OptRot & \underline{$50.11$} & $13.97$ & $0.43$ & $57.67$ & $10.34$ & $0.362$ & \underline{$65.27$} & $8.38$ & $0.348$ \\
        \midrule
        OptRot-v2 & $49.61$ & $13.83$ & $0.437$ & $57.36$ & $10.6$ & $0.406$ & $59.31$ & $12.1$ & $0.568$\\
        \OptRotH & $48.48$ & $14.45$ & $0.458$ & \underline{$58.05$} & $10.54$ & $0.37$ & $64.84$ & $8.5$ & $0.344$\\
        \OptRotH-v2 & $49.92$ & $14.13$ & $0.426$ & $58.02$ & $10.6$ & $0.353$ & $63.94$ & $8.74$ & $0.37$ \\
        \bottomrule
    \end{tabular}
\end{table}

\begin{table}[h!]
    \centering
    \caption{Results for 3-bit weight-only quantization with GPTQ.}
    \label{tab:3bit_performance}
    \begin{tabular}{l@{\hskip 0.7cm}ccccccccc}
        \toprule
        \multirow{2}{*}{\textbf{Method}} & \multicolumn{3}{c}{Llama-3.2-1B} & \multicolumn{3}{c}{Llama-3.2-3B} & \multicolumn{3}{c}{Llama-3.1-8B} \\
        \cmidrule(lr){2-4} \cmidrule(lr){5-7} \cmidrule(lr){8-10}
        & Acc $\uparrow$ & Wiki $\downarrow$ & KL $\downarrow$ & Acc $\uparrow$ & Wiki $\downarrow$ & KL $\downarrow$ & Acc $\uparrow$ & Wiki $\downarrow$ & KL $\downarrow$ \\
        \midrule
        FP16 & $56.77$ & $9.76$ & $0$ & $64.72$ & $7.81$ & $0$ & $70.29$ & $6.24$ & $0$\\
        \midrule
        No Rotation & $40.23$ & $77.32$ & $1.62$ & $50.55$ & $15.58$ & $0.75$ & $58.13$ & $11.82$ & $0.64$ \\
        Quarot & $51.25$ & $14.47$ & $0.427$ & \underline{$60.68$} & $10.24$ & $0.322$ & $66.36$ & $8$ & $0.275$ \\
        SpinQuant & $50.8$ & $14.38$ & $0.415$ & $59.83$ & $10.28$ & $0.327$ & \underline{$66.81$} & $7.99$ & $0.275$ \\
        \midrule
        OptRot (top-50) & \underline{$51.77$} & $14$ & $0.407$ & \underline{$60.68$} & $10.1$ & $0.32$ & $66.8$ & $7.86$ & $0.263$ \\
        OptRot & $51.37$ & $13.78$ & $0.384$ & $60.36$ & $10.17$ & $0.313$ & $\textbf{67.31}$ & $7.85$ & \underline{$0.257$} \\
        \midrule
        OptRot-v2 & $51.52$ & $13.66$ & $0.387$ & $60.44$ & \underline{$10.05$} & \underline{$0.309$} & $63.41$ & $9.91$ & $0.398$\\
        \OptRotH & $\textbf{51.86}$ & $\textbf{13.58}$ & $\textbf{0.368}$ & $59.03$ & $\textbf{10.03}$ & $0.312$ & $66.22$ & $\textbf{7.78}$ & $\textbf{0.253}$\\
        \OptRotH-v2 & $50.28$ & \underline{$13.65$} & \underline{$0.369$} & $\textbf{60.87}$ & $10.06$ & $\textbf{0.297}$ & $66.28$ & \underline{$7.83$} & $0.258$\\
        \bottomrule
    \end{tabular}
\end{table}

\begin{table}[h!]
    \centering
    \caption{Results for 4-bit weight-only quantization with RTN.}
    \label{tab:rtn}
    \begin{tabular}{l@{\hskip 0.7cm}ccccccccc}
        \toprule
        \multirow{2}{*}{\textbf{Method}} & \multicolumn{3}{c}{Llama-3.2-1B} & \multicolumn{3}{c}{Llama-3.2-3B} & \multicolumn{3}{c}{Llama-3.1-8B} \\
        \cmidrule(lr){2-4} \cmidrule(lr){5-7} \cmidrule(lr){8-10}
        & Acc $\uparrow$ & Wiki $\downarrow$ & KL $\downarrow$ & Acc $\uparrow$ & Wiki $\downarrow$ & KL $\downarrow$ & Acc $\uparrow$ & Wiki $\downarrow$ & KL $\downarrow$ \\
        \midrule
        FP16 & $56.77$ & $9.76$& $0$ & $64.72$& $7.81$& $0$& $70.29$& $6.24$ & $0$\\
        \midrule
        No Rotation & $48.14$ & $15.13$ & $0.451$ & $54.84$ & $16.2$ & $0.416$ & $64.38$ & $8.23$ & $0.279$ \\
        Quarot & $51.27$ & $13.7$ & $0.4$ & $59.41$ & $10.1$ & $0.327$ & $67.2$ & $7.57$ & $0.244$ \\
        SpinQuant & $51.18$ & $13.56$ & $0.399$ & $59.54$ & $10.03$ & $0.324$ & $67.14$ & $7.57$ & $0.242$ \\
        OptRot & $52.26$ & $12.81$ & $0.331$ & $60.46$ & $9.64$ & $0.273$ & $67.65$ & $7.4$ & $0.224$ \\
        SpinQuant (W4) & $\textbf{54.65}$ & $\textbf{11.75}$ & $\textbf{0.241}$ & $\textbf{61.75}$ & $\textbf{8.97}$ & $\textbf{0.198}$ & $\textbf{67.79}$ & $\textbf{7.14}$ & $\textbf{0.171}$ \\
        \bottomrule
    \end{tabular}
\end{table}

\begin{figure}[ht!]
    \centering
    \includegraphics[width=\linewidth]{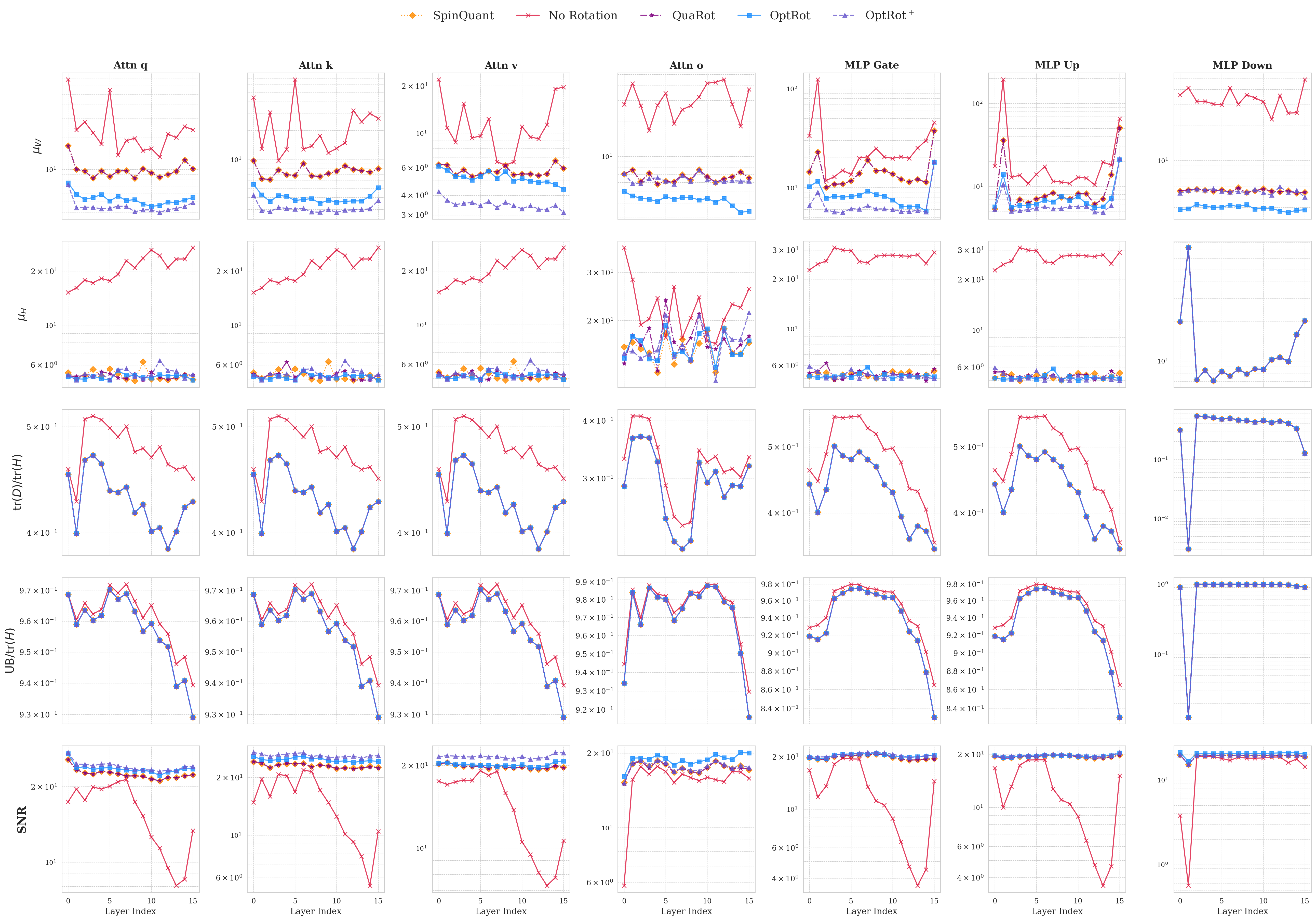}
    \caption{Weight incoherence $\mu_W$ optimized by OptRot (top row), Hessian incoherence $\mu_H$ (second row), $\tr(D) / \text{Tr}(H)$ (third row), $\text{UB} / \text{Tr}(H)$ (fourth row), and the SNR after quantization with GPTQ (bottom row) for Llama-3.2-1B.}
    \label{fig:incoherence-all-1B}
\end{figure}
\begin{figure}[ht!]
    \centering
    \includegraphics[width=\linewidth]{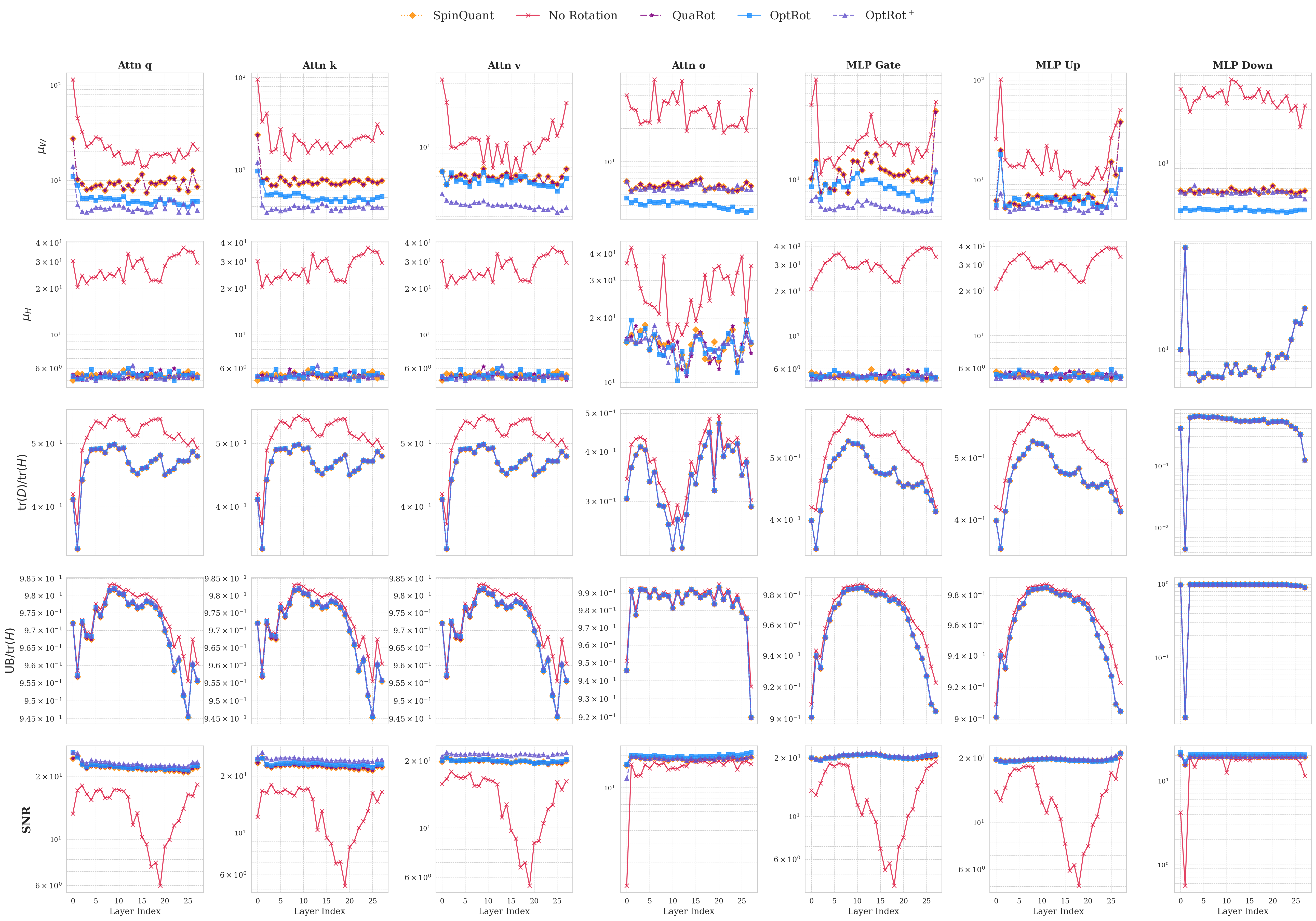}
    \caption{Weight incoherence $\mu_W$ optimized by OptRot (top row), Hessian incoherence $\mu_H$ (second row), $\tr(D) / \text{Tr}(H)$ (third row), $\text{UB} / \text{Tr}(H)$ (fourth row), and the SNR after quantization with GPTQ (bottom row) for Llama-3.2-3B.}  
    \label{fig:incoherence-all-3B}
\end{figure}
\begin{figure}[ht!]
    \centering
    \includegraphics[width=\linewidth]{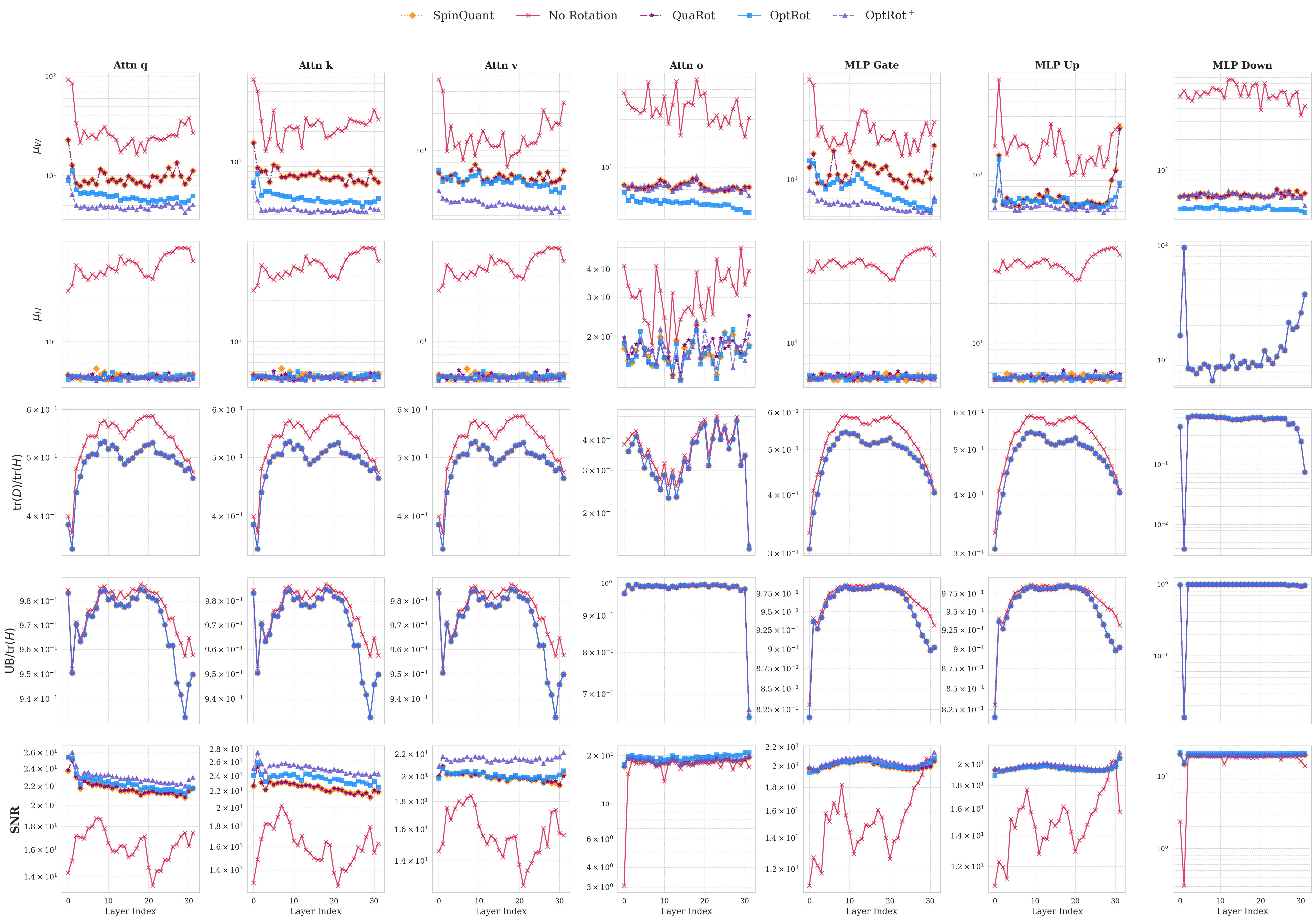}
    \caption{Weight incoherence $\mu_W$ optimized by OptRot (top row), Hessian incoherence $\mu_H$ (second row), $\tr(D) / \text{Tr}(H)$ (third row), $\text{UB} / \text{Tr}(H)$ (fourth row), and the SNR after quantization with GPTQ (bottom row) for Llama-3.1-8B.}
    \label{fig:incoherence-all-8B}
\end{figure}

\begin{figure}[ht!]
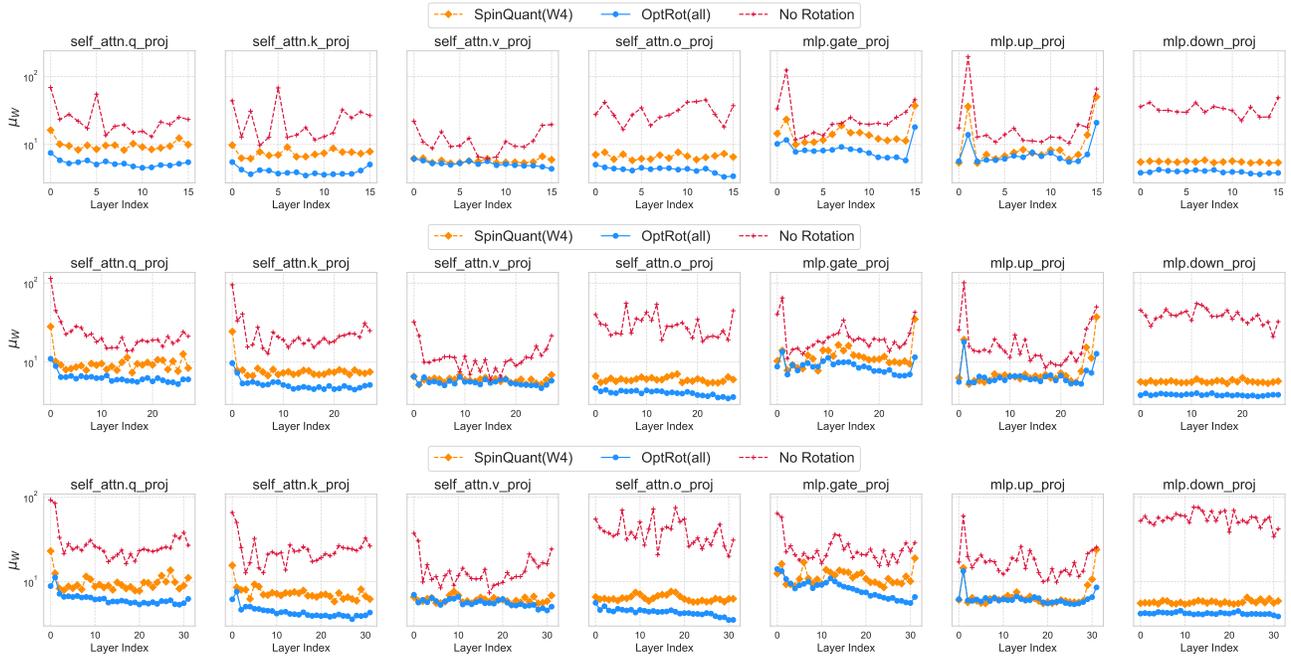

    \centering
    \includegraphics[width=\linewidth]{figs/rtn_mu_Llama-1B.png}
    \includegraphics[width=\linewidth]{figs/rtn_mu_Llama-3B.png}
    \includegraphics[width=\linewidth]{figs/rtn_mu_Llama-8B.png}
    \caption{Weight incoherence $\mu_W$ comparison for SpinQuant (W4) and OptRot 
    on Llama-3.2-1B (top row), Llama-3.2-3B (middle row) and Llama-3.1-8B (bottom row).}
    \label{fig:mu-rtn}
\end{figure}

\begin{figure}[ht!]
    \centering
    \includegraphics[width=\linewidth]{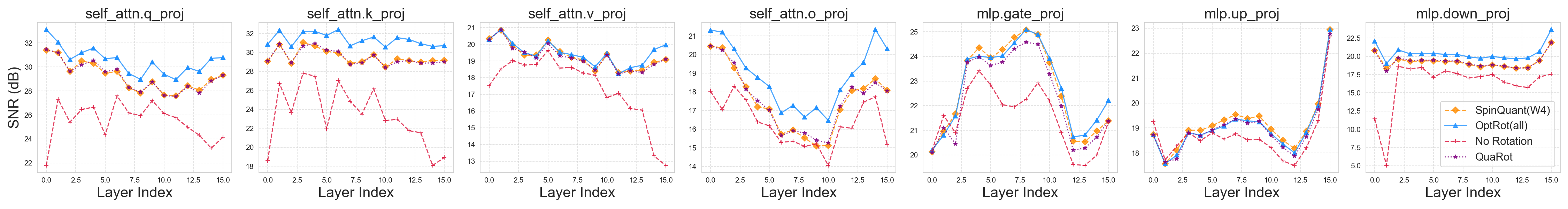}
    \includegraphics[width=\linewidth]{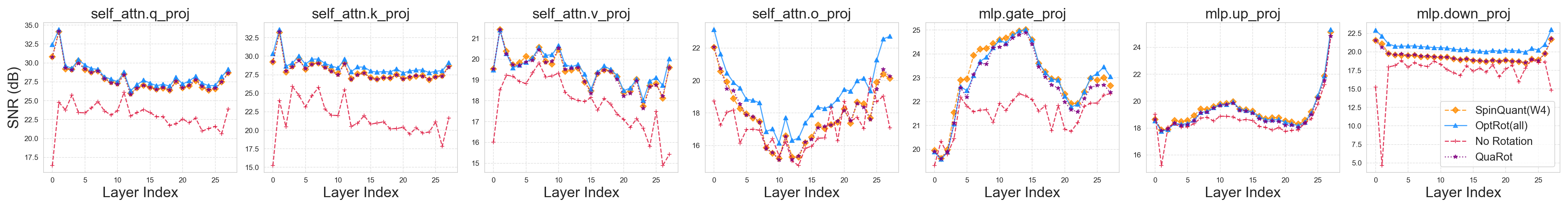}
    \includegraphics[width=\linewidth]{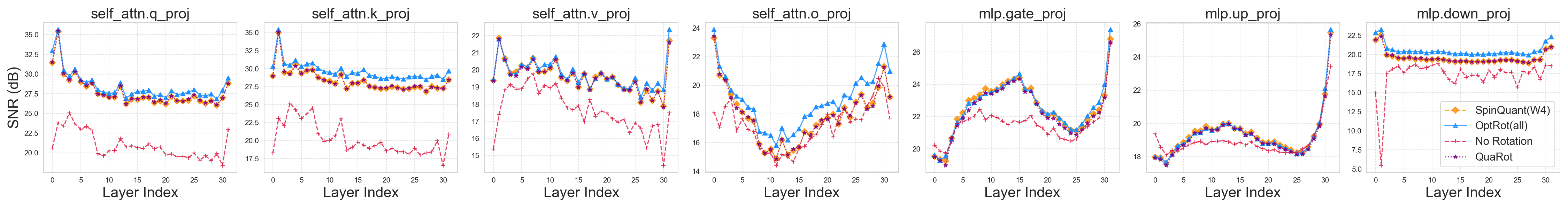}
    \caption{SNR for weight-only quantization with RTN at 4-bits for Llama-3.2-1B (top row), Llama-3.2-3B (middle row) and Llama-3.1-8B (bottom row).}
    \label{fig:hessian-snr-rtn}
\end{figure}

\begin{table}[ht!]
    \centering
    \caption{4-bit weight-only quantization with GPTQ on Qwen2 models.}
    \label{tab:weight-only-qwen2}
    \setlength{\tabcolsep}{4pt} 
    \begin{tabular}{l@{\hskip 0.5cm}ccc@{\hskip 0.5cm}ccc@{\hskip 0.5cm}ccc}
        \toprule
        \multirow{2}{*}{\textbf{Method}} & \multicolumn{3}{c}{Qwen2-0.5B} & \multicolumn{3}{c}{Qwen2-1.5B} & \multicolumn{3}{c}{Qwen2-7B} \\
        \cmidrule(lr){2-4} \cmidrule(lr){5-7} \cmidrule(lr){8-10}
         & Acc $\uparrow$ & Wiki $\downarrow$ & KL $\downarrow$ & Acc $\uparrow$ & Wiki $\downarrow$ & KL $\downarrow$ & Acc $\uparrow$ & Wiki $\downarrow$ & KL $\downarrow$ \\
        \midrule
        FP16 & $49.3$ & $13.37$ & $0$ & $58.87$ & $9.5$ & $0$ & $68.26$ & $7.09$ & $0$\\
        \midrule
        No Rotation & $43.46$ & $18$ & $0.302$ & $54.95$ & $11.62$ & $0.174$ & $66.36$ & $7.87$ & $0.09$\\
        QuaRot & $48.3$ & $\textbf{14.25}$ & $\textbf{0.073}$ & $\textbf{58.18}$ & $\textbf{9.93}$ & $0.049$ & $\textbf{68.3}$ & $7.34$ & $0.027$\\
        SpinQuant & $47.1$ & $\textbf{14.25}$ & $0.082$ & $57.95$ & $\textbf{9.93}$ & $0.049$ & $67.79$ & $7.34$ & $0.026$ \\
        \midrule
        OptRot (all) & $\textbf{49.13}$ & $\textbf{14.25}$ & $0.085$ & $58.01$ & $\textbf{9.93}$ & $\textbf{0.041}$ & $68.04$ & $\textbf{7.28}$ & $\textbf{0.023}$\\
        \bottomrule
    \end{tabular}
\end{table}

\end{document}